\documentclass[letterpaper]{article} 
\usepackage{aaai2026}  
\usepackage{times}  
\usepackage{helvet}  
\usepackage{courier}  
\usepackage[hyphens]{url}  
\usepackage{graphicx} 
\usepackage{natbib}  
\usepackage{caption} 
\usepackage{algorithm}
\usepackage{algorithmic}

\usepackage{newfloat}
\usepackage{listings}
\DeclareCaptionStyle{ruled}{labelfont=normalfont,labelsep=colon,strut=off} 
\floatstyle{ruled}
\newfloat{listing}{tb}{lst}{}
\floatname{listing}{Listing}
\usepackage{xcolor}
\usepackage{amsfonts}
\usepackage{amsmath}
\usepackage{multirow}
\usepackage{booktabs}
\newtheorem{theorem}{Theorem}

\usepackage[table]{xcolor}
\usepackage{amssymb}
\usepackage{pifont}
\usepackage{paralist}

\usepackage{pgfplots}
\pgfplotsset{compat=1.17}
\usepackage{tikz}
\usetikzlibrary{positioning, arrows.meta, calc, fit, decorations.pathreplacing, decorations.markings}

\usepackage{caption}
\usepackage{subfig}

\newcommand{\cmark}{\textcolor{green!60!black}{\ding{51}}}  
\newcommand{\xmark}{\textcolor{red}{\ding{55}}}             

\title{Kernelized Edge Attention: Addressing Semantic Attention Blurring in Temporal Graph Neural Networks}
\author{
    Govind Waghmare\textsuperscript{\rm 1},
    Srini Rohan Gujulla Leel\textsuperscript{\rm 1},
    Nikhil Tumbde\equalcontrib\textsuperscript{\rm 1},\\
    Sumedh B G\equalcontrib\textsuperscript{\rm 1},
    Sonia Gupta\textsuperscript{\rm 1},
    Srikanta Bedathur\textsuperscript{\rm 2}
}
\affiliations{
    \textsuperscript{\rm 1}Mastercard, 
    \textsuperscript{\rm 2}Indian Institute of Technology, Delhi\\

    \{govind.waghmare, srinirohan.gujullaleel, nikhil.tumbde, sumedh.bg, sonia.gupta\}@mastercard.com,
    srikanta@cse.iitd.ac.in
}

\begin{document}

\maketitle



\begin{abstract}
Temporal Graph Neural Networks (TGNNs) aim to capture the evolving structure and timing of interactions in dynamic graphs. Although many models incorporate time through encodings or architectural design, they often compute attention over entangled node and edge representations, failing to reflect their distinct temporal behaviors. Node embeddings evolve slowly as they aggregate long-term structural context, while edge features reflect transient, timestamped interactions (e.g. messages, trades, or transactions). This mismatch results in semantic attention blurring, where attention weights cannot distinguish between slowly drifting node states and rapidly changing, information-rich edge interactions. As a result, models struggle to capture fine-grained temporal dependencies and provide limited transparency into how temporal relevance is computed. This paper introduces KEAT (Kernelized Edge Attention for Temporal Graphs), a novel attention formulation that modulates edge features using a family of continuous-time kernels, including Laplacian, RBF, and learnable MLP variant. KEAT preserves the distinct roles of nodes and edges, and integrates seamlessly with both Transformer-style (e.g., DyGFormer) and message-passing (e.g., TGN) architectures. It achieves up to 18\% MRR improvement over the recent DyGFormer and 7\% over TGN on link prediction tasks, enabling more accurate, interpretable and temporally aware message passing in TGNNs.
\end{abstract}

\begin{links}
    \link{Code}{https://waghmaregovind.github.io/KEAT-TemporalGNN}
\end{links}

\section{Introduction}

Learning from time-evolving graph data is critical to applications such as recommendation, event forecasting, and fraud detection. While Graph Neural Networks (GNNs) have proven effective on static graphs \cite{kipf2017semisupervised,graphsage,liao2018lanczosnet}, adapting them to continuous-time dynamics poses fundamental challenges~\cite{kazemi2020representation,longa2023graph,zheng2025survey}. TGNNs address this by modeling both structural and temporal dependencies, often through specialized architectures like temporal walks~\cite{wang2021inductivecawn}, sequence patching ~\cite{dygformer2023}, and alternative message-passing ~\cite{luo2022neighborhoodnat,cong2023dographmixer}, or via time encodings~\cite{chen2025rethinking,Xu2020Inductivetgat,xu2019self}.

\begin{figure}[t]
\centering
\begin{tikzpicture}[
    font=\footnotesize,
    every node/.style={font=\footnotesize},
    profile/.style={draw, fill=blue!10, rounded corners=2pt, minimum width=2.2cm, minimum height=0.6cm},
    txn/.style={draw, fill=red!20, rounded corners=2pt, minimum width=.8cm, minimum height=0.4cm},
    axis/.style={very thick, ->},
    label/.style={font=\scriptsize}
]

\draw[axis] (0,0) -- (7.0,0) node[below right] {Time};

\draw[thick, blue, smooth, domain=0.5:7, samples=100] plot(\x, {0.8 + 0.1*sin(deg(\x)) + 0.02*\x});

\node[profile, anchor=south west] at (0.5,1.5) {Cardholder Profile};

\node[txn, anchor=south west] at (0.5,0.2) {Txn};

\draw[->, thick, blue] (3.2,1.7) -- (6.5,1.2) node[midway, above, sloped] {\small gradual change};

\foreach \x in {1.5, 2.4, 4.2, 4.3, 5.1, 5.5} {
    \draw[red!70!black, thick] (\x,0) -- (\x,1);
}

\draw[decorate,decoration={brace,mirror, amplitude=4pt}] (0.4,2.3) -- (7.0,2.3) 
    node[midway, above=5pt, text width=7.5cm, align=center, font=\small] 
    {Slow evolution of node features (e.g., credit score, user intent)};

\draw[decorate,decoration={brace, amplitude=4pt}] (1.2,-0.3) -- (5.8,-0.3) 
    node[midway, below=5pt, text width=7cm, align=center, font=\small]
    {Rapid, irregular edge activity (e.g., transactions)};

\end{tikzpicture}

\caption{Temporal mismatch in dynamic graphs. While user profiles (nodes) change gradually (e.g., credit behavior), interactions like transactions (edges) vary rapidly and irregularly. Attention mechanisms in existing TGNNs often entangle these signals, leading to semantic attention blurring.}

\label{fig:temporal-mismatch}
\end{figure}
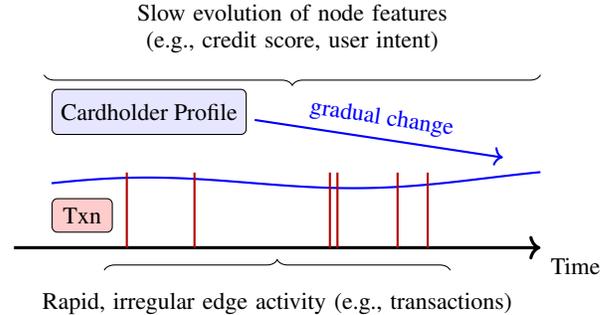

To model temporal relationships in graph, attention-based TGNNs have gained popularity, particularly those inspired by Transformer architectures \cite{SimpleDyG,dygformer2023,tgn_icml_grl2020}. These models extend self-attention to graphs by assigning learnable importance scores to neighboring nodes and edges. While effective, most existing approaches compute attention over entangled node and edge representations. They fail to explicitly modulate edge features based on complex temporal dynamics, overlooking the fact that nodes and edges often evolve at different temporal rates. See Figure~\ref{fig:temporal-mismatch} for an illustration of this mismatch. 

For instance, in financial transaction graphs, cardholder profiles (nodes) evolve gradually, reflecting changes like creditworthiness or spending behavior. Whereas, transaction records (edges) vary rapidly and irregularly, especially in the presence of suspicious or high-frequency activity. Capturing these short-term edge patterns without overwhelming them with slowly drifting node context is crucial for temporal precision. In particular, these models do not directly encode how time should modulate attention weights or influence message aggregation. This leads to \emph{semantic attention blurring}, where slowly evolving node embeddings, capturing long-term structural context, are mixed indiscriminately with sparse, transient edge features that carry rich, recent information. As a result, attention mechanisms struggle to prioritize temporally and semantically relevant interactions, reducing both temporal fidelity and interpretability.


While prior works focus on improving time encodings to address these issues, such efforts may offer limited returns under current attention formulations, such as those implemented in \texttt{TransformerConv}\footnote{Documentation for \texttt{TransformerConv}: \url{https://pytorch-geometric.readthedocs.io/en/stable/generated/torch_geometric.nn.conv.TransformerConv.html} \cite{pytorch_geo_fey}}, where attention scores are computed over the sum of node and edge projections. This formulation, or slight variations of it, is adopted in models such as TGN~\cite{tgn_icml_grl2020}, DyGFormer~\cite{dygformer2023}, and others~\cite{Xu2020Inductivetgat}, which similarly suffer from an inability to disentangle temporally distinct signals. As a result, such aggregation blurs temporal distinctions, making it difficult to separate recent, informative edge events from older or semantically unrelated interactions. This phenomenon is illustrated in Figure \ref{fig:intro_attn_plots}. In contrast, we shift the focus from time encoding design to the attention mechanism itself. By modulating only edge features, \textbf{comprising raw edge attributes and time encodings}, with continuous-time kernels, we introduce explicit temporal awareness into the attention computation and message aggregation. Crucially, our formulation is compatible with, and often enhances, existing time encoding schemes, making it broadly applicable and encoding-agnostic.

\begin{figure*}[t]
    \centering
    \includegraphics[width=0.99\linewidth]{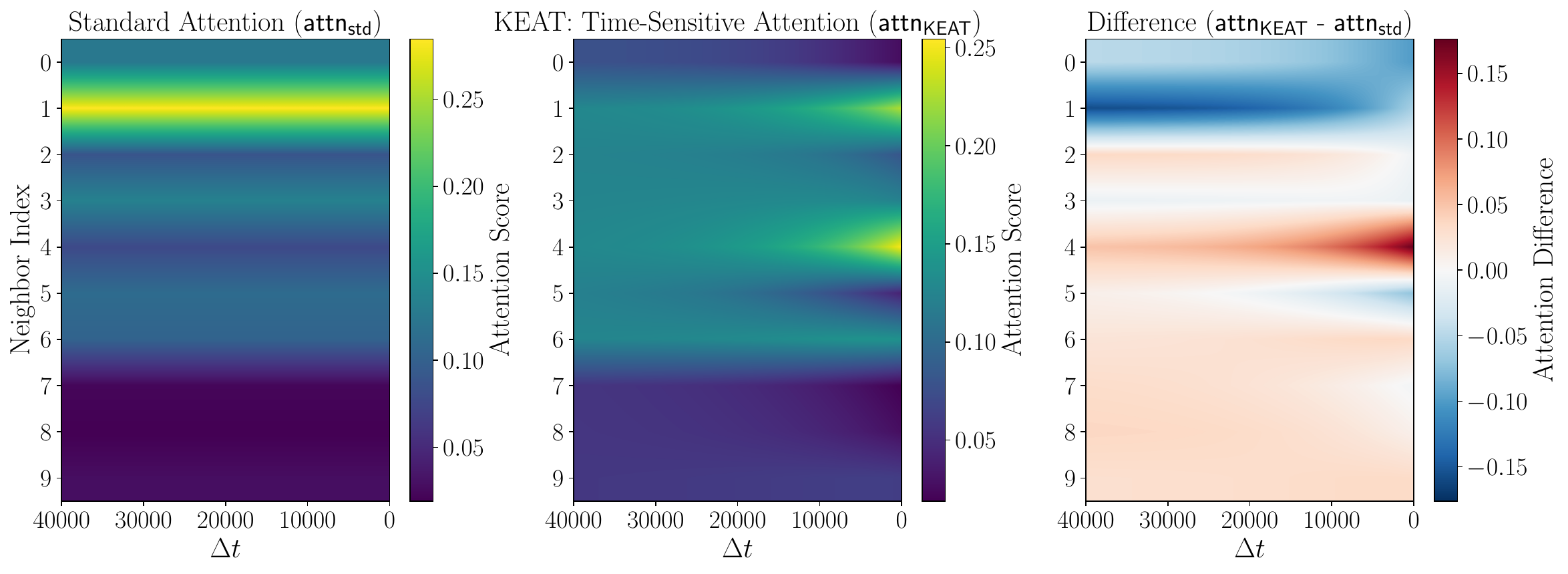}
    \caption{Temporal attention on \texttt{tgbl-wiki} illustrating \textbf{semantic attention blurring}. We plot attention over ten neighbors across relative time $\Delta t$. (Left) Standard attention assigns uniform attention values across time for each neighbor, leading to semantic blurring. This shows invariance of attention to time encodings or $\Delta t$. (Middle) KEAT introduces time-aware modulation, emphasizing recent (smaller $\Delta t $ value) and contextually relevant edges. (Right) The difference highlights how KEAT corrects the blurred focus of standard attention. See experimental section for details.}
    \label{fig:intro_attn_plots}
\end{figure*}


In this work, we propose a novel temporal attention formulation that directly addresses the mismatch in temporal dynamics between node and edge features in TGNNs. We call our framework \textbf{KEAT} \textit{(\textbf{\underline{K}}ernelized \textbf{\underline{E}}dge \textbf{\underline{A}}ttention for \textbf{\underline{T}}emporal Graphs)}. KEAT applies a family of continuous-time kernels exclusively to edge features, thereby incorporating temporal sensitivity into attention without altering node semantics. The key idea is simple yet effective: edge-time features are modulated via kernel functions such as Laplacian, RBF, or a learnable MLP-based variant. This modulation is applied within the key and value projections of attention, allowing recent interactions to be emphasized while preserving the relatively stable, slow-evolving role of node embeddings. KEAT integrates seamlessly with Transformer-style and message-passing TGNNs and remains agnostic to specific time encoding schemes.

Existing models like DyGFormer~\cite{dygformer2023} leverage temporal patches and time-interval encodings to capture evolving graph dynamics. However, they still rely on conventional attention formulations and inherit similar limitations. We show that integrating KEAT into DyGFormer and other TGNN architectures improves temporal precision, interpretability, and predictive performance. These gains validate the broad applicability and plug-and-play nature of KEAT across diverse temporal modeling settings. Our key contributions are summarized as follows:

\begin{itemize}
\item We formally identify the issue of \emph{semantic attention blurring} in Transformer-based TGNNs, where attention scores over combined node and edge projections hinder temporal precision and interpretability.

\item To resolve this, we introduce KEAT, a time-aware attention mechanism that applies continuous-time kernel modulation exclusively to edge features, preserving the distinct temporal roles of nodes and edges.

\item KEAT is agnostic to time encoding schemes and compatible with a range of TGNN architectures. It requires minimal architectural changes and no additional supervision, making it practical for integration into existing systems.

\item Through extensive experiments on dynamic graph benchmarks, we demonstrate that KEAT improves accuracy, enhances temporal fidelity, and produces interpretable attention patterns aligned with evolving edge dynamics.
\end{itemize}

\section{Related Work}

\paragraph{Temporal Graph Neural Networks.}
TGNNs extend GNNs to dynamic graphs by modeling both structure and time. Early methods such as TGN~\cite{tgn_icml_grl2020}, TGAT~\cite{Xu2020Inductivetgat}, and CAWN~\cite{wang2021inductivecawn} incorporate time using continuous-time encodings or time-difference embeddings. Other approaches adopt positional or sinusoidal encodings inspired by language models~\cite{xu2019self}. Recent work like LeTE~\cite{chen2025rethinking} introduces a learnable transformation framework for time encoding using nonlinear mappings such as Fourier and splines, aiming to unify and generalize existing encoding schemes. Despite these advances, most of these models assume that better time encodings alone will yield better performance. In contrast, we argue that time encoding improvements have limited effect if the attention formulation itself cannot separate the dynamics of node and edge features. Our work departs from this direction by keeping time encodings fixed and focusing instead on making attention computations explicitly time-aware through kernel-based modulation.

\paragraph{Attention Mechanisms in TGNNs.}
Transformer-inspired TGNNs such as TGAT~\cite{Xu2020Inductivetgat}, TGN~\cite{tgn_icml_grl2020}, and DyGFormer~\cite{dygformer2023} use attention to weigh the influence of temporal neighbors during message passing. These models often follow the \texttt{TransformerConv} formulation, a standardized implementation from the PyTorch Geometric, computing the attention scores from the sum of projected node and edge features. This design simplifies computation but introduces semantic attention blurring, wherein temporally distinct messages are mixed. Some methods attempt to refine attention using structural bias (e.g., CAWN’s walk-based decay~\cite{wang2021inductivecawn}), and a few incorporate time signals into attention projections through time encodings~\cite{Xu2020Inductivetgat}, but explicit time-aware modulation, such as continuous-time scaling of attention inputs, remains underexplored. Our kernel-based attention addresses this gap by decoupling edge dynamics from node semantics via time-dependent scaling, providing a lightweight yet effective way to inject temporal information into the attention mechanism.


\paragraph{Modulated Attention Beyond TGNNs.}
Time-aware attention has been studied in other domains such as NLP, time series and credit risk modeling. In language models, temporal self-attention augments Transformer-based architectures with timestamped text ~\cite{rosin2022temporal}. 
Powerformer \cite{powerformer} and \cite{niu2024attentionrobustrepresentationtime} treat attention as the primary representation for time series. In credit risk assessment, DGNN-SR~\cite{credit_risk_multi_view_te} integrates multiple time perspectives directly into the attention module. In static graphs, Gradformer \cite{gradformer} applies exponential decay to graph transformer. Works like \cite{alibi,Chi2022KERPLEKR} study attention in NLP for extrapolation. While these works make attention more expressive, our approach introduces a modular, temporally kernel-based attention mechanism that applies time-conditioned scaling exclusively to edge representations, thus decoupling node and edge dynamics. This formulation allows time to play a more explicit role in the attention computation and helps reduce semantic attention blurring in a more interpretable way.

\paragraph{Temporal Information Modeling.}
Several TGNNs incorporate temporal information through auxiliary mechanisms rather than within attention. JODIE~\cite{kumar2019predictingjodie} applies time-conditioned projections to evolve user embeddings, and DyRep~\cite{trivedi2018dyrep} models temporal interactions via point processes intensities. TGN~\cite{tgn_icml_grl2020} maintains time-aware memory modules, while CAWN~\cite{wang2021inductivecawn} integrates temporal signals during neighbor sampling. In contrast, our approach embeds temporal characteristics directly into the attention computation by modulating edge features with kernel-weighted time encodings. This enables interpretable, differentiable, and architecture-agnostic modulation.

\paragraph{Modular and Extendable TGNN Designs.}
Recent architectures such as DyGFormer~\cite{dygformer2023} adopt patch-based attention blocks to process temporally localized subgraphs. While these designs improve scalability and temporal resolution, they still compute attention using mixed node-edge projections. Our kernel-based formulation can be seamlessly integrated into such architectures, enhancing their temporal fidelity without disrupting their tokenization, encoding, or aggregation pipelines. This highlights the flexibility of our method as a plug-and-play module for improving temporal sensitivity in a wide range of TGNNs.

\section{Methodology}

\subsection{Problem Setting and Notation}
Let $\mathcal{G} = (\mathcal{V}, \mathcal{E}, T)$ be a dynamic graph, where $\mathcal{V}$ is the set of nodes, $\mathcal{E} \subseteq \mathcal{V} \times \mathcal{V}$ is the set of timestamped edges, and $T$ is the set of event times. Each interaction is a tuple $(i, j, \mathbf{e}_{ij}, t_{ij})$, with $i, j \in \mathcal{V}$, $\mathbf{e}_{ij} \in \mathbb{R}^{d_e}$ representing raw edge features, and $t_{ij} \in \mathbb{R}_+$ denoting the interaction time. TGNNs typically capture temporal context by computing the time difference $\Delta t = t_i - t_{ij}$ between the current query time $t_i$ and a past interaction time $t_{ij}$. This $\Delta t$ is encoded using a time encoding function $\phi: \mathbb{R}_+ \rightarrow \mathbb{R}^{d_t}$, producing $\phi(\Delta t)$. The temporal and raw edge features are then concatenated to form: $\bar{\mathbf{e}}_{ij} = [\mathbf{e}_{ij} \, \| \, \phi(\Delta t)] \in \mathbb{R}^{d_e + d_t}$. Given the temporal neighborhood $\mathcal{N}(i)$ of node $i$, the objective is to compute updated node embeddings $\mathbf{h}_i' \in \mathbb{R}^d$ that are both structurally informative and temporally sensitive.

\subsection{Transformer-style Temporal Message Passing}
Transformer-based attention mechanisms are widely adopted in TGNNs for aggregating messages from temporal neighbors. A representative implementation, \texttt{TransformerConv}~\cite{pytorch_geo_fey}, computes attention from a source node $i$ to a target node $j$ as:
\begin{equation}
\alpha_{ij} = \text{softmax}_j \left( \frac{ (\mathbf{W}_q \mathbf{h}_i)^\top (\mathbf{W}_k \mathbf{h}_j + \mathbf{W}_e \bar{\mathbf{e}}_{ij}) }{ \sqrt{d_k} } \right),
\label{eq:standard_attention}
\end{equation}
where $\mathbf{h}_i, \mathbf{h}_j \in \mathbb{R}^d$ are node embeddings, $\bar{\mathbf{e}}_{ij} \in \mathbb{R}^{d_e + d_t}$ is the concatenated edge-time feature, and $\mathbf{W}_q, \mathbf{W}_k \in \mathbb{R}^{d' \times d}$, $\mathbf{W}_e \in \mathbb{R}^{d' \times (d_e + d_t)}$ are learnable projection matrices. The attention is scaled by $d_k = d'$ as in standard Transformers \cite{NIPS2017_3f5ee243_vaswani}. The final embedding of node $i$ is:
\begin{equation}
\mathbf{h}_i' = \mathbf{W}_{\text{self}} \mathbf{h}_i + \sum_{j \in \mathcal{N}(i)} \alpha_{ij} \cdot (\mathbf{W}_v \mathbf{h}_j + \mathbf{W}_e' \bar{\mathbf{e}}_{ij}),
\label{eq:standard_update}
\end{equation}
where $\mathbf{W}_v \in \mathbb{R}^{d' \times d}$, $\mathbf{W}_e' \in \mathbb{R}^{d' \times (d_e + d_t)}$, and $\mathbf{W}_{\text{self}} \in \mathbb{R}^{d' \times d}$ are learnable weight matrices. This formulation, or variants of it, is used in several TGNNs such as TGN~\cite{tgn_icml_grl2020}, DyGFormer~\cite{dygformer2023}, and others.

\subsection{Time Encodings in Temporal Graphs}

Popular time encoding functions $\phi(\Delta t)$ include:

\begin{itemize}
    \item \textbf{Fixed encodings}: GraphMixer~\cite{cong2023dographmixer} employs non-trainable sinusoidal time encodings, similar to the positional encodings used in standard Transformers.

    \item \textbf{Relative and learnable encodings}: TGN~\cite{tgn_icml_grl2020} uses learnable sinusoidal embeddings for relative time, similar to TGAT~\cite{Xu2020Inductivetgat}, applied in both message computation and memory updates. LeTE~\cite{chen2025rethinking} learns flexible time encoding functions using splines and Fourier-based transformations.
\end{itemize}

\paragraph{Limitations of Time Encodings.}
Time encodings $\phi(\Delta t)$ in TGNNs are often concatenated with node or edge features (see Eq.\ref{eq:standard_attention} and \ref{eq:standard_update}) but do not explicitly influence the attention weights. As a result, they weakly capture temporal patterns like frequency or semantic relevance (Figure \ref{fig:intro_attn_plots}), limiting the model’s ability to differentiate diverse temporal interactions.

A common approach is to encode time differences $\Delta t$ using sinusoidal functions such as $\phi(\Delta t) = \cos(\omega \Delta t)$ or $\sin(\omega \Delta t)$. As shown in Appendix B, these simple periodic encodings capture only a limited subset of the moment spectrum of the inter-arrival distribution $p(\Delta t)$:
\begin{equation}
\mathbb{E}_{\Delta t}[\cos(\omega \Delta t)] = \sum_{n=0}^{\infty} \frac{(-1)^n \omega^{2n}}{(2n)!} \mathbb{E}[(\Delta t)^{2n}].
\end{equation}
While some prior works use both sine and cosine terms to capture more statistical moments \cite{chen2025rethinking}, these encodings remain passive and fail to directly shape attention weights. As a result, they often miss temporal asymmetries important for aggregation. This limitation is exacerbated in real-world temporal graphs, where the distribution $p(\Delta t)$ often shifts across training and validation sets (refer Appendix C). In many graphs, early interactions are typically sparse, while later ones are denser, causing skewed distributions and shifts in both the mean and higher-order moments.

KEAT addresses this challenge by introducing kernel-modulated encodings of the form $f(\Delta t) = \psi(\Delta t) \cdot \bar{\mathbf{e}}_{ij}$, where $\psi(\Delta t)$ is a temporal kernel. This formulation expands into a full polynomial over $\Delta t$ and becomes sensitive to all moments of $p(\Delta t)$ (see Appendix B.2). Importantly, the modulated signal $f(\Delta t)$ is used to directly influence attention computation, ensuring that time information shapes both the neighbor weighting and the final embedding output.

\subsection{Design Criteria for Temporal Kernels}
To enable time-aware attention, we modulate edge features using a continuous temporal kernel $\psi(\Delta t)$ that scales contributions based on the elapsed time $\Delta t$. The kernel should satisfy several desirable properties: (i) it should decay monotonically with increasing $\Delta t$ to prioritize recent events; (ii) it must be bounded within $[0,1]$ to ensure numerical stability; (iii) it should be continuous to support smooth optimization; and (iv) it should be flexible. Parameterized MLPs can model richer, potentially non-monotonic (vs. Laplacian and RBF) temporal patterns when interpretability is less critical.

\subsection{Family of Temporal Kernels}
We instantiate the temporal modulation function $\psi(\Delta t)$ using the following kernel families:

\begin{itemize}
    \item \textbf{Laplacian kernel:} $\psi(\Delta t) = \exp\left( -\frac{\Delta t}{\sigma} \right)$
    \item \textbf{RBF kernel:} $\psi(\Delta t) = \exp\left( -\frac{\Delta t^2}{\sigma^2} \right)$
    \item \textbf{Learned kernel (MLP):} $\psi(\Delta t) = \mathrm{MLP}(\Delta t)$
\end{itemize}

Here, $\sigma$ denotes the standard deviation of inter-event time differences computed from the training set. The Laplacian and RBF kernels are parameter-free and introduce interpretable temporal biases by smoothly attenuating the influence of older interactions. In contrast, the MLP-based kernel offers greater flexibility, enabling the model to learn complex, data-driven modulation patterns. All variants are differentiable and integrate seamlessly into end-to-end training.

\subsection{Method: KEAT – Kernelized Attention over Time}
We introduce KEAT, a mechanism that integrates temporal modulation into attention by applying a kernel $\psi(\Delta t)$ to the edge-time feature $\bar{\mathbf{e}}_{ij}$ \textit{before} it influences attention or embedding updates. In Transformer-style architectures (Eq.~\ref{eq:standard_attention}) attention scores are computed from projected node embeddings and edge features. Node embeddings evolve slowly and capture long-term structural information, while edge features encode temporally localized, interaction-specific signals. Thus, modulating the edge component using a continuous-time kernel $\psi(\Delta t)$ provides a principled way to inject temporal sensitivity. KEAT requires only a lightweight modification to edge projections, introducing negligible overhead during training or inference.

KEAT explicitly biases or modulates the attention toward different temporal interactions by scaling edge-time features via $\psi(\Delta t)$. The modified attention score becomes:
\begin{equation}
\alpha_{ij} = \text{softmax}_j \left( \frac{ (\mathbf{W}_q \mathbf{h}_i)^\top \left( \mathbf{W}_k \mathbf{h}_j + \mathbf{W}_e \left[f(\Delta t) \right] \right) }{ \sqrt{d_k} } \right),
\label{eq:modulated_attention}
\end{equation}
where, $f(\Delta t) = \psi(\Delta t) \cdot \bar{\mathbf{e}}_{ij}$ and the corresponding node embedding update formula is:
\begin{equation}
\mathbf{h}_i' = \mathbf{W}_{\text{self}} \mathbf{h}_i + \sum_{j \in \mathcal{N}(i)} \alpha_{ij} \cdot \left( \mathbf{W}_v \mathbf{h}_j + \mathbf{W}_e' \left[f(\Delta t) \right] \right)
\label{eq:modulated_update}
\end{equation}


\noindent This design ensures that temporally closer interactions contribute significantly to attention computation and feature aggregation, while preserving the original Transformer architecture. Notably, node features remain unaltered, allowing a clean decoupling of structural and temporal contributions.

In DyGFormer~\cite{dygformer2023}, interaction histories are divided into patches, each summarizing a local neighborhood around source or destination nodes. To integrate KEAT, we modulate only the edge features within each patch using a representative patch timestamp (e.g., mean edge time). Specifically, the query is scaled by $\exp(t_{\text{patch}})$ and the key by $\exp(-t_{\text{patch}})$, resulting in attention scores of the form $\exp(t_{\text{query}} - t_{\text{key}})$. This introduces a directional temporal bias, enabling time-sensitive attention without altering DyGFormer's architecture. See Appendix~D for details.

\begin{table*}[t]
    \centering
    \resizebox{0.77\textwidth}{!}{%
    \begin{tabular}{lccccc}
    \toprule
    \textbf{Method} & \textbf{Split} & \texttt{tgbl-wiki} & \texttt{tgbl-review} & \texttt{tgbl-coin} & \texttt{tgbl-comment} \\
    \midrule

    \multirow{2}{*}{$\text{EdgeBank}_{\text{tw}}$} 
    & Val  & 0.600 & 0.024 & 0.492 & 0.124 \\
    & Test & 0.571 & 0.025 & 0.580 & 0.149 \\

    \midrule

    \multirow{2}{*}{$\text{EdgeBank}_\ensuremath{\infty}$} 
    & Val  & 0.527 & 0.023 & 0.315 & 0.109 \\
    & Test & 0.495 & 0.023 & 0.359 & 0.129 \\

    \midrule

    \multirow{2}{*}{DyRep} 
    & Val  & 0.072 \scriptsize{$\pm$ 0.009} & 0.216 \scriptsize{$\pm$ 0.031} & 0.512 \scriptsize{$\pm$ 0.014} & 0.291 \scriptsize{$\pm$ 0.028} \\
    & Test & 0.050 \scriptsize{$\pm$ 0.017} & 0.220 \scriptsize{$\pm$ 0.030} & 0.452 \scriptsize{$\pm$ 0.046} & 0.289 \scriptsize{$\pm$ 0.033} \\

    \midrule
    
    \multirow{2}{*}{TNCN} 
        & Val  & 0.741 \scriptsize{$\pm$ 0.001} & 0.325 \scriptsize{$\pm$ 0.003} & 0.740 \scriptsize{$\pm$ 0.002} & 0.643 \scriptsize{$\pm$ 0.003} \\
        & Test & \cellcolor{gray!20}0.718 \scriptsize{$\pm$ 0.001} & 0.377 \scriptsize{$\pm$ 0.010} & \cellcolor{yellow!20}0.762 \scriptsize{$\pm$ 0.004} & \cellcolor{yellow!20}0.697 \scriptsize{$\pm$ 0.006} \\
    
    \midrule
    
    \multirow{2}{*}{CTAN} 
    & Val  & \textemdash & \textemdash & \textemdash & \textemdash \\
    & Test & 0.668 \scriptsize{$\pm$ 0.007} & \cellcolor{yellow!20}0.405 \scriptsize{$\pm$ 0.004} & 0.748 \scriptsize{$\pm$ 0.004} & \cellcolor{gray!20}0.671 \scriptsize{$\pm$ 0.067} \\

    \midrule
    
    \multirow{2}{*}{TGN} & Val  & 0.435 \scriptsize{$\pm$ 0.069} & 0.313 \scriptsize{$\pm$ 0.012} & 0.607 \scriptsize{$\pm$ 0.014} & 0.356 \scriptsize{$\pm$ 0.019} \\
                     & Test & 0.396 \scriptsize{$\pm$ 0.060} & 0.349 \scriptsize{$\pm$ 0.020} & 0.586 \scriptsize{$\pm$ 0.037} & 0.379 \scriptsize{$\pm$ 0.021} \\

    \cmidrule{2-6}
    \multirow{2}{*}{KEAT-TGN} & Val  & 0.515 \scriptsize{$\pm$ 0.053} & 0.324 \scriptsize{$\pm$ 0.004} & 0.624 \scriptsize{$\pm$ 0.012} & 0.368 \scriptsize{$\pm$ 0.019} \\
                                  & Test & 0.474 \scriptsize{$\pm$ 0.031} & \cellcolor{gray!20}0.380 \scriptsize{$\pm$ 0.007} & 0.632 \scriptsize{$\pm$ 0.023} & 0.400 \scriptsize{$\pm$ 0.020} \\

    \cmidrule{2-6}
    

    \multirow{1}{*}{\shortstack[l]{Improvement}} 
    & Test & +7.76\% & +3.10\% & +4.62\% & +2.10\% \\
    
    \midrule
    \multirow{2}{*}{DyGFormer} 
        & Val  & 0.816 \scriptsize{$\pm$ 0.005} & 0.219 \scriptsize{$\pm$ 0.017} & 0.730 \scriptsize{$\pm$ 0.002} & 0.613 \scriptsize{$\pm$ 0.003} \\
        & Test &  \cellcolor{yellow!20}0.798 \scriptsize{$\pm$ 0.004} & 0.224 \scriptsize{$\pm$ 0.015} & \cellcolor{gray!20}0.752 \scriptsize{$\pm$ 0.004} & 0.670 \scriptsize{$\pm$ 0.001} \\
    \cmidrule{2-6}
    \multirow{2}{*}{KEAT-DyGFormer} 
        & Val  & {0.829} \scriptsize{$\pm$ 0.003} & {0.335} \scriptsize{$\pm$ 0.020} & \textemdash & \textemdash \\
        & Test & \cellcolor{green!20}{0.815} \scriptsize{$\pm$ 0.005} &  \cellcolor{green!20}{0.412} \scriptsize{$\pm$ 0.012} &  \cellcolor{green!20}{0.803} \scriptsize{$\pm$ 0.003} &  \cellcolor{green!20}{0.776} \scriptsize{$\pm$ 0.001} \\
    \cmidrule{2-6}

    \multirow{1}{*}{\shortstack[l]{Improvement}} 
        & Test & +1.69\% & +18.80\% & +5.10\% & +10.60\% \\
    
    \bottomrule
    \end{tabular}
    }
     \caption{
        Link prediction results (MRR) on the TGBL benchmark datasets. We report both validation and test MRR for each method. All baseline results are taken from TGB leaderboard \cite{TGB}. \textbf{KEAT} denotes our kernel-based attention mechanism integrated into TGN and DyGFormer backbones. The top three test results for each dataset are highlighted using \colorbox{green!20}{green} (first), \colorbox{yellow!20}{yellow} (second), and \colorbox{gray!20}{gray} (third). Entries marked with \textemdash\ indicate that baseline results are not available.}
    \label{tab:link_pred_results}
    \end{table*}

\subsection{Advantages of Kernel-Based Attention}

KEAT offers several desirable properties:
\begin{inparaitem}
\item \textbf{Temporal precision}: (Laplacian and RBF kernels) emphasizes recent interactions by downweighting distant ones;
\item \textbf{Semantic disentanglement}: isolates evolving edge patterns while preserving node semantics;
\item \textbf{Architectural generality}: seamlessly integrates into diverse TGNNs (e.g., TGN, TGAT, DyGFormer);
\item \textbf{Simplicity and efficiency}: lightweight, differentiable, and requires no additional preprocessing.
\end{inparaitem}

Beyond demonstrating empirical utility, we provide theoretical justification for using temporal kernels in attention. In particular, we show that kernel-based modulation suppresses contributions from higher-order moments of the time distribution, improving robustness to distribution shifts across training and inference. The following theorem formally captures this effect for the Laplacian or RBF kernel $\psi(t)$:

\begin{theorem}
Let \( p(t) \) be a probability density function supported on \( [0, \infty) \) such that \( \mathbb{E}[t^n] < \infty \) and \( \mathbb{E}[\psi(t) t^n] < \infty \) for all \( n \geq 0 \), where \( \psi(t) \) is a non-negative, monotonically decreasing kernel function. Let \( \phi(t) = \sum_{n=0}^{\infty} c_n t^n \) be an analytic time encoding function with coefficients \( \{c_n\} \). Define the kernel-to-base ratio for each moment order \( n \) as:
\[
R_n = \frac{\mathbb{E}[\psi(t) t^n]}{\mathbb{E}[t^n]}, \quad \text{with } R_0 = \mathbb{E}[\psi(t)].
\]
Then, the sequence \( \{R_n\}_{n=0}^\infty \) is strictly decreasing and converges to zero: \( \lim_{n \to \infty} R_n = 0 \). As a result, the expected kernel-weighted time encoding \( \mathbb{E}[\psi(t) \cdot \phi(t)] \) becomes increasingly dominated by lower-order terms\footnote{This formulation naturally extends to expectations of the form \( \mathbb{E}[\psi(t) \cdot \bar{\mathbf{e}}_{ij}] \), as used in Equations~\ref{eq:modulated_attention} and~\ref{eq:modulated_update}.} .
\end{theorem}

The proof is provided in Appendix~E. In addition, we establish a complementary result: incorporating temporal kernels into attention reduces the variance of the attention logits. This promotes more stable neighbor weighting, particularly under temporally imbalanced or skewed interaction histories. This variance-reduction effect further underscores the robustness and interpretability benefits of kernel-based attention. See Appendix~E for a complete derivation.

\section{Experiments}

\begin{table*}[t]
\centering
\resizebox{0.9\textwidth}{!}{%
\begin{tabular}{lccccccc}
\toprule
Method & Kernel & \multicolumn{2}{c}{\texttt{tgbl-wiki}} & \multicolumn{2}{c}{\texttt{tgbl-review}} & \multicolumn{2}{c}{\texttt{tgbl-coin}}
\\
\cmidrule(lr){3-4} \cmidrule(lr){5-6} \cmidrule(lr){7-8}
&& Val & Test & Val & Test & Val & Test \\
\midrule
\multirow{3}{*}{KEAT-TGN} & Laplacian   & 0.515 \scriptsize{ $\pm$ 0.053} & \textbf{0.474} \scriptsize{ $\pm$ 0.031} & \textbf{0.324} \scriptsize{ $\pm$ 0.004} & \textbf{0.378} \scriptsize{ $\pm$ 0.007} & \textbf{0.624} \scriptsize{ $\pm$ 0.012} & 0.632 \scriptsize{ $\pm$ 0.023} \\
 & RBF         & 0.512 \scriptsize{ $\pm$ 0.027} & 0.437 \scriptsize{ $\pm$ 0.035} & 0.322 \scriptsize{ $\pm$ 0.008} & 0.374 \scriptsize{ $\pm$ 0.011} & 0.623 \scriptsize{ $\pm$ 0.015} & \textbf{0.636} \scriptsize{ $\pm$ 0.021} \\
& MLP         & \textbf{0.524} \scriptsize{ $\pm$ 0.017} & 0.456 \scriptsize{ $\pm$ 0.027} & \textbf{0.324} \scriptsize{ $\pm$ 0.007} & 0.372 \scriptsize{ $\pm$ 0.008} & 0.620 \scriptsize{ $\pm$ 0.022} & 0.617 \scriptsize{ $\pm$ 0.024} \\
\midrule
TGN & w/o Kernel  & 0.435 \scriptsize{ $\pm$ 0.069} & 0.396 \scriptsize{ $\pm$ 0.060} & 0.313 \scriptsize{ $\pm$ 0.012} & 0.349 \scriptsize{ $\pm$ 0.020} & 0.607 \scriptsize{ $\pm$ 0.014} & 0.586 \scriptsize{ $\pm$ 0.037} \\
\bottomrule
\end{tabular}
}
\caption{MRR results on three TGBL datasets using different kernel types. Each cell reports mean $\pm$ standard deviation for validation and test sets. Best results in each column (based on mean) are shown in bold.}
\label{tab:kernel_ablations}
\end{table*}

We empirically validate the effectiveness of our proposed temporal kernel modulation framework across multiple dynamic graph learning tasks. Our evaluation focuses on three key aspects:
(1) compatibility and integration with existing attention-based TGNN architectures, 
(2) the ability to capture fine-grained temporal dependencies through moment-aware attention modulation, and
(3) performance gains under varying temporal encodings and graph dynamics.

Experiments are primarily conducted on the Temporal Graph Benchmark (TGB)~\cite{TGB}, which provides large-scale, realistic datasets with diverse temporal characteristics. To further evaluate the generality and robustness of our kernelized attention mechanism, we extend the benchmark suite by incorporating additional datasets from the JODIE~\cite{kumar2019predictingjodie} framework for dynamic link prediction, as well as the DGraphFin dataset~\cite{DGraphFin} for dynamic node classification. This extended evaluation suite enables us to test KEAT across a wider range of temporal graph settings, including financial graphs, and user-item interactions.

KEAT is implemented as a lightweight, modular enhancement to existing Transformer-based TGNNs without requiring architectural changes. We report results using standardized splits and evaluation metrics to isolate the impact of temporal attention modulation. Unless stated otherwise, we use the Laplacian kernel by default. A detailed description of datasets along with statistics is provided in Appendix C.

\subsection{Task Definitions}
We consider two standard tasks on temporal graphs $\mathcal{G}$. In \textbf{dynamic link prediction}, the objective is to estimate the probability of a link between two nodes at a given time, framed as a ranking problem with Mean Reciprocal Rank (MRR) as the  metric. For JODIE datasets, we measure performance using AUC and Average Precision (AP). In \textbf{dynamic node classification}, the goal is to predict a node’s label at a specific time based on its interaction history and temporal context, capturing evolving patterns such as user preferences or risk. Performance is evaluated using NDCG@10 for \texttt{tgbn} dasets and using AUC for the DGraphFin dataset.

\paragraph{Baselines}
Our proposed kernel-based attention mechanism is designed to enhance Transformer-style TGNNs. To evaluate its impact, we integrate it into two widely used attention-based architectures, TGN and DyGFormer, serving as our primary backbones. The resulting models, KEAT-TGN and KEAT-DyGFormer, are compared against their vanilla counterparts to isolate the effect of our contribution. We also report results from several strong baselines, including DyRep~\cite{trivedi2018dyrep}, TNCN~\cite{zhang2024efficienttncn}, and CTAN~\cite{CTAN}, as well as memory-based methods EdgeBank$_{\text{tw}}$ and EdgeBank$_{\infty}$~\cite{poursafaei2022towardsedgebank}, included in the TGB leaderboard. While some of these baselines do not explicitly model temporal attention, they provide valuable context for evaluating our method. Other baselines, such as TCL ~\cite{tcl_wang_2021}, NAT \cite{luo2022neighborhoodnat}, CAWN \cite{wang2021inductivecawn}, and GraphMixer \cite{cong2023dographmixer} underperform compared to DyGFormer and/or TGB and are evaluated on a narrower set of datasets, as shown in the TGB leaderboard\footnote{\url{https://tgb.complexdatalab.com/docs/leader_linkprop/}}. Appendix~F summarizes baseline selection criteria, ranking on the link prediction task, their scalability and dataset coverage.

\paragraph{Implementation Details.}
We implement our model using \texttt{PyTorch} and \texttt{PyTorch Geometric}, and the official TGB framework \cite{TGB}. All experiments follow the default hyperparameters provided by the TGB benchmark to ensure fair and reproducible comparisons. Each experiment is run five times with different random seeds. The experimentation settings and hyperparameters for JODIE datasets and \texttt{DGraphFin} are provided in Appendix~G.

\paragraph{Link Prediction Results.} Table~\ref{tab:link_pred_results} presents the MRR for link prediction on the \texttt{tgbl} benchmark datasets. Our method, Kernel-based Attention (KEAT), is applied to two popular TGNN backbones: TGN and DyGFormer. Across all datasets, KEAT variants demonstrate consistent improvements over their respective base models. Notably, KEAT-DyGFormer achieves the best test performance on all four datasets, \texttt{tgbl-wiki}, \texttt{tgbl-review}, \texttt{tgbl-coin}, and \texttt{tgbl-comment}, demonstrating the generality and effectiveness of our kernelized attention mechanism. KEAT-TGN also shows gains over TGN on all datasets, with improvements ranging from +2.1\% to +7.8\% in test MRR. These results affirm that augmenting TGNNs with time-sensitive, structure-aware attention yields measurable and consistent performance benefits. KEAT also demonstrates consistently stronger performance on the JODIE datasets across multiple evaluation metrics, highlighting its effectiveness in handling time-evolving interactions. Details are in Appendix G.




\paragraph{Effect of Kernel Choice.}
We evaluate the effect of different kernel choices within the {KEAT} framework through an ablation study using three variants: \textit{Laplacian}, \textit{RBF}, \textit{MLP}, with a baseline (no kernel modulation). All experiments are conducted using the TGN backbone on three benchmark datasets, \texttt{tgbl-wiki}, \texttt{tgbl-review}, and \texttt{tgbl-coin} (see Table~\ref{tab:kernel_ablations}). The results consistently indicate that temporal kernels enhance performance. For instance, on \texttt{tgbl-wiki}, the Laplacian kernel yields an MRR of {0.474{\scriptsize$\pm$0.031}}, marking a significant improvement of {+7.8\%} over the no-kernel baseline ({0.396{\scriptsize$\pm$0.060}}). Similarly, the \textit{RBF kernel} achieves the highest MRR on \texttt{tgbl-coin} at {0.636{\scriptsize$\pm$0.021}}, surpassing the baseline by {+5.0\%}. These gains underscore the benefit of temporally modulating edge features when computing attention.

Laplacian and RBF perform most consistently, benefiting from positive-definite, distance-based temporal modulation. MLP is more expressive but has higher variance, as it lacks built-in inductive bias and depends heavily on available data. Overall, these results highlight the generalization advantage of structured kernels in time-sensitive tasks.

\paragraph{Choosing the Right Kernel.}
Laplacian prioritizes recency, ideal for short-term temporal dynamics. In contrast, the RBF kernel provides a smoother, more balanced decay that captures both recent and mid-range dependencies. MLP is suited for high-data regimes with complex, non-monotonic temporal patterns but may underperform when supervision is limited. Overall, Laplacian and RBF offer strong, low-variance performance across settings. Unless otherwise specified, we use Laplacian as the default kernel.


\paragraph{Node Classification Results.} KEAT consistently improves node classification performance when integrated into TGN across diverse \texttt{tgbn} datasets (refer Appendix G). On \texttt{tgbn-trade} and \texttt{tgbn-genre}, KEAT-TGN achieves substantial gains of +5.30\% and +5.29\% in test NDCG@10, respectively. These improvements highlight KEAT’s ability to inject temporal structure into attention mechanism, aligning node representations more closely with their evolving interaction context. Even on challenging datasets like \texttt{tgbn-token}, KEAT improves performance by +1.1\%, indicating robustness. The slight but consistent gain on \texttt{tgbn-reddit} suggests that KEAT complements existing signals without degrading strong baselines. These findings reinforce KEAT’s effectiveness in stabilizing node-level predictions and improving temporal generalization. Refer Appendix G for results on \texttt{DGraphFin}.

\paragraph{Impact of KEAT Across Time Encoding Strategies.}
We evaluate the robustness of KEAT with various time encoding methods, including \texttt{GraphMixer} (non-learnable), \texttt{TGN/TGAT}(learnable), and \texttt{LeTE}(learnable, Fourier and Spline-based), on three datasets. Results in Appendix G (on \texttt{tgbl-wiki}, \texttt{tgbl-review} and \texttt{tgbl-coin}) demonstrate that KEAT consistently improves test MRR, regardless of the encoding strategy. For instance, on \texttt{tgbl-wiki}, KEAT improves test MRR by +7.8\% with \texttt{TGN/TGAT} and +7.3\% with \texttt{LeTE}. Even when the raw encodings are strong (e.g., \texttt{GraphMixer} on \texttt{tgbl-review}), KEAT improves or maintains performance. Importantly, we observe reduced variance across runs when KEAT is used, suggesting it helps regularize attention with structured inductive bias. These results indicate that KEAT is orthogonal to the choice of time encoding and acts as a plug-and-play enhancement for diverse temporal message-passing models.

Attention plots in Appendix G highlights how KEAT enables edge-modulated, time-sensitive attention over neighbors. Unlike standard attention, which yields flat, temporally agnostic weights, KEAT dynamically adjusts attention based on both edge timing and feature context. This allows attention to vary meaningfully with $\Delta t$, yielding sharper, context-aware focus that mitigates semantic attention blurring and enhances temporal discrimination. Furthermore, attention weight heatmaps over time (refer to Appendix~G) illustrate the interpretability of KEAT: they reveal how attention shifts gradually as edge timestamps grow older, offering insights into how the model prioritizes recent and semantically relevant interactions.


\paragraph{Computational Complexity.} With Laplacian or RBF kernels, KEAT incurs no additional computational overhead and retains the same time and space complexity as its underlying backbone architectures, TGN and DyGFormer. MLP kernel operates on scalar time gaps with a few parameters. Under standard TGBL settings, we observe similar inference complexity of MLP kernel compared to other kernels.

\paragraph{Extended results.} For completeness, we provide detailed dataset statistics, extended ablations (e.g., node vs. edge modulation, kernel width sensitivity), results on additional datasets (including JODIE and DGraphFin), full experimental settings, and theoretical proofs in Appendix. We encourage readers to refer to these sections for further insights. Together, these empirical results establish KEAT as a general-purpose enhancement for TGNNs. By embedding temporal structure directly into the attention computation, KEAT improves both predictive performance and training stability across datasets, time encoding schemes, and kernel types.

\paragraph{Limitations of KEAT.}
While KEAT demonstrates strong performance across datasets, it may be less effective in scenarios when edge activity is extremely sparse or highly clustered in time, which can limit the informativeness of temporal modulation. Addressing such challenging regimes remains an important and promising direction for future work.

\section{Conclusion}

We presented KEAT, a kernel-based attention framework that addresses semantic attention blurring in TGNNs. By introducing temporal kernels into the attention mechanism, KEAT preserves fine-grained distinctions between recent and older edge interactions, enabling more precise temporal reasoning. The method is lightweight, model-agnostic, and compatible with diverse time encoding strategies. Empirical results on diverse datasets show consistent performance gains across multiple tasks and model backbones such as TGN and DyGFormer. Extensive ablations further confirm the impact of kernel choice, edge-level modulation, and reduced attention variance. Overall, KEAT offers a simple yet effective solution to improve both the accuracy and generalization of TGNNs with minimal overhead.

\section{Acknowledgments}
Srikanta Bedathur acknowledges his DS Chair Professor of AI fellowship and the funding support from Mastercard AI Garage for this work. 

\bibliography{aaai2026}

\twocolumn[\section{Appendix}

\noindent\textbf{Appendix A:} Notations

\noindent\textbf{Appendix B:} Moment Analysis of Kernel-Weighted Time Encodings

\noindent\textbf{Appendix C:} Dataset Description and Statistics

\noindent\textbf{Appendix D:} DygFormer

\noindent\textbf{Appendix E:} Impact of Temporal Kernels on Higher-Order Moments

\noindent\textbf{Appendix F:} Baselines

\noindent\textbf{Appendix G:} Additional Experimental Details and Results

\noindent\textbf{Appendix H:} Code and Reproducibility]

\begin{table*}[ht!]
\centering

\begin{tabular}{ll}
\multicolumn{2}{c}{\Large \textbf{Appendix A: Notations}} \\
\toprule
\textbf{Symbol} & \textbf{Description} \\
\midrule
\multicolumn{2}{l}{\textit{Graph and Temporal Setting}} \\
$\mathcal{G} = (\mathcal{V}, \mathcal{E}, T)$ & Dynamic temporal graph \\
$\mathcal{V}, \mathcal{E}$ & Node and edge sets \\
$T$ & Set of event timestamps \\
$t_{ij}$ & Timestamp of edge $(i, j)$ \\
$\Delta t = t_i - t_{ij}$ & Time elapsed since interaction \\
$\mathcal{N}(i)$ & Temporal neighbors of node $i$ \\
$\mathbf{e}_{ij} \in \mathbb{R}^{d_e}$ & Static edge features \\
$\bar{\mathbf{e}}_{ij} = [\mathbf{e}_{ij} \| \phi(\Delta t)]$ & Concatenated edge-time feature \\
\midrule
\multicolumn{2}{l}{\textit{Embeddings and Attention Computation}} \\
$\mathbf{h}_i, \mathbf{h}_j \in \mathbb{R}^d$ & Node embeddings \\
$\alpha_{ij}$ & Attention from node $i$ to $j$ \\
$d_k$ & Key/query dimensionality \\
$\mathbf{h}_i'$ & Updated embedding of node $i$ \\
$\mathbf{W}_q, \mathbf{W}_k, \mathbf{W}_v$ & Query, key, and value projection matrices \\
$\mathbf{W}_{\text{self}}$ & Self-embedding residual weight \\
$\mathbf{W}_e, \mathbf{W}_e'$ & Edge feature transformation matrices \\
\midrule
\multicolumn{2}{l}{\textit{Temporal Modulation (KEAT)}} \\
$\psi(\Delta t)$ & Temporal kernel applied to edge \\
$\psi(\Delta t) = \exp(-\Delta t / \sigma)$ & Laplacian kernel \\
$\psi(\Delta t) = \exp(-\Delta t^2 / \sigma^2)$ & RBF kernel \\
$\psi(\Delta t) = \mathrm{MLP}(\Delta t)$ & Learned kernel (parametric) \\
$\sigma$ & Temporal scale parameter (std. of $\Delta t$) \\
$\psi(\Delta t) \cdot \bar{\mathbf{e}}_{ij}$ & Modulated edge-time feature \\
\midrule
\multicolumn{2}{l}{\textit{Series Interpretation and Expectations}} \\
$\omega$ & Frequency parameter in cosine kernel \\
$\mathbb{E}_{\Delta t}[\cos(\omega \Delta t)]$ & Expected time-frequency response \\
$\mathbb{E}[(\Delta t)^{2n}]$ & $2n$-th moment of time differences \\
$n$ & Index in Taylor expansion \\
\bottomrule
\end{tabular}
\caption{Summary of Notation}
\label{tab:notation}
\end{table*}

\clearpage

\begin{table*}[t]
    \centering
    
    \begin{tabular}{lllrrrrr}
    \toprule
    \textbf{Task} & \textbf{Scale} & \textbf{Dataset} & \textbf{\#Nodes} & \textbf{\#Edges} & \textbf{\#Steps} & \textbf{Surprise} & \textbf{\# Negatives}\\
    \midrule
    \multirow{5}{*}{\shortstack{Link\\(\texttt{tgbl})}}
    & Small  & \texttt{tgbl-wiki}    & 9,227     & 157,474     & 152,757    & 0.108  & 999\\
    & Small  & \texttt{tgbl-review}  & 352,637   & 4,873,540   & 6,865      & 0.987  & 100\\
    & Medium & \texttt{tgbl-coin}    & 638,486   & 22,809,486  & 1,295,720  & 0.120  & 20\\
    & Large  & \texttt{tgbl-comment} & 994,790   & 44,314,507  & 30,998,030 & 0.823  & 20\\
    & Large  & \texttt{tgbl-flight}  & 18,143    & 67,169,570  & 1,385      & 0.024  & 20\\
    \bottomrule
    \end{tabular}

    \caption{Dataset statistics from TGB across link prediction and node classification tasks. Surprise values are taken from \cite{TGB}. \# Negatives is the number of negatives used for one positive link in the MRR calculation.}
    \label{tab:tgbl_stats}
    \end{table*}

\begin{table}[t]
\centering
\begin{tabular}{lrrc}
\toprule
\textbf{Dataset} & \textbf{\#Nodes} & \textbf{\#Edges} & \textbf{\# Negatives}\\
\midrule
\texttt{jodie-reddit}     & 6,509  & 25,470  & 1  \\
\texttt{jodie-wiki}  & 9,227  & 157,474 & 1\\
\texttt{jodie-mooc}       & 7,144  & 411,749 & 1\\
\texttt{jodie-lastfm}     & 1,980  & 1,293,103 & 1\\
\bottomrule
\end{tabular}
\caption{Statistics of \texttt{jodie-datasets} used for link prediction. For these datasets, one negative link is used against one positive link, making it binary classification problem.}
\label{tab:jodie-stats}
\end{table}

\begin{table*}[t]
    \centering
    
    \begin{tabular}{lllrrrc}
    \toprule
    \textbf{Task} & \textbf{Scale} & \textbf{Dataset} & \textbf{\#Nodes} & \textbf{\#Edges} & \textbf{\#Steps} & \textbf{Surprise} \\
    \midrule
    \multirow{4}{*}{\shortstack{Node\\(\texttt{tgbn})}}
    & Small  & \texttt{tgbn-trade}   & 255       & 468,245     & 32         & 0.023  \\
    & Medium & \texttt{tgbn-genre}   & 1,505     & 17,858,395  & 133,758    & 0.005   \\
    & Large  & \texttt{tgbn-reddit}  & 11,766    & 27,174,118  & 21,889,537 & 0.013   \\
    & Large  & \texttt{tgbn-token}   & 61,756    & 72,936,998  & 2,036,524  & 0.014   \\ 
    \bottomrule
    \end{tabular}

    \caption{Dataset statistics from TGB across node classification tasks.  Surprise values are taken from \cite{TGB}.}
    \label{tab:tgbn_stats}
    \end{table*}

\begin{table}[t]
\centering
\begin{tabular}{lrr}
\toprule
\textbf{Dataset} & \textbf{\#Nodes} & \textbf{\#Edges} \\
\midrule
\texttt{DGraphFin}     & 3,700,550  & 4,300,999    \\
\bottomrule
\end{tabular}
\caption{Statistics of \texttt{DGraphFin} used for node classification.}
\label{tab:dgraphfin-stats}
\end{table}

\section{Appendix B: Moment Analysis of Kernel-Weighted Time Encodings}
\label{app:a_moments}

\subsection{B.1 Taylor Expansion of $\cos(\omega \Delta t)$}

We begin by analyzing a commonly used time encoding function, $\cos(\omega \Delta t)$, where $\omega$ is a frequency hyperparameter and $\Delta t$ is the time difference. The Taylor series expansion is given by:
\begin{equation}
\cos(\omega \Delta t) = \sum_{n=0}^{\infty} \frac{(-1)^n (\omega \Delta t)^{2n}}{(2n)!}
\end{equation}

Taking expectation with respect to the inter-arrival distribution $p(\Delta t)$, we obtain:
\begin{equation}
\mathbb{E}_{\Delta t}[\cos(\omega \Delta t)] = \sum_{n=0}^{\infty} \frac{(-1)^n \omega^{2n}}{(2n)!} \mathbb{E}[(\Delta t)^{2n}]
\end{equation}

This reveals that the encoding is only sensitive to even-order moments (e.g., variance, kurtosis), and entirely insensitive to odd-order moments like the mean or skewness.

\subsection{B.2 Introducing Temporal Modulation Kernels}

To enrich moment sensitivity, we consider multiplying the encoding with a temporal modulation kernel $\psi(\Delta t)$, such as:

\begin{itemize}
    \item \textbf{Laplacian kernel:} $\psi(\Delta t) = \exp(-|\Delta t|/\tau)$
\end{itemize}

Consider the kernel-modulated encoding:
\begin{equation}
f(\Delta t) = \psi(\Delta t) \cdot \cos(\omega \Delta t)
\end{equation}

Both $\psi(\Delta t)$ and $\cos(\omega \Delta t)$ can be expressed as Taylor series:
\begin{align}
\psi(\Delta t) &= \sum_{m=0}^{\infty} \frac{(-\lambda)^m}{m!} (\Delta t)^m \quad \text{(e.g., exponential)} \\
\cos(\omega \Delta t) &= \sum_{n=0}^{\infty} \frac{(-1)^n \omega^{2n}}{(2n)!} (\Delta t)^{2n}
\end{align}

Multiplying the two series yields:
\begin{equation}
f(\Delta t) = \left( \sum_{m=0}^{\infty} a_m (\Delta t)^m \right) \cdot \left( \sum_{n=0}^{\infty} b_n (\Delta t)^{2n} \right),
\end{equation}

\begin{equation}
f(\Delta t) = \sum_{k=0}^{\infty} c_k (\Delta t)^k,
\end{equation}
where $a_m = \frac{(-\lambda)^m}{m!}$ and $b_n = \frac{(-1)^n \omega^{2n}}{(2n)!}$ are the series coefficients of the kernel and encoding, respectively. The resulting coefficients $c_k$ arise from all pairs $(m, n)$ satisfying $k = m + 2n$, leading to a series expansion that includes both even and odd powers of $\Delta t$.

Taking expectation:
\begin{equation}
\mathbb{E}_{\Delta t}[f(\Delta t)] = \sum_{k=0}^{\infty} c_k \cdot \mathbb{E}[(\Delta t)^k]
\end{equation}
This confirms that the modulated encoding incorporates contributions from all statistical moments of the inter-arrival distribution $p(\Delta t)$.

\subsection{B.3 Implications for Temporal Graph Models}

This analysis highlights a fundamental limitation of standard encoding functions such as $\cos(\omega \Delta t)$ or $\sin(\omega \Delta t)$: they are inherently blind to half of the moment spectrum. By introducing a modulation kernel $\psi(\Delta t)$, we recover full moment sensitivity, enabling the encoding to reflect both symmetric and asymmetric properties of the temporal distribution.

This hybrid formulation enhances the robustness of attention mechanisms to distributional shifts in $\Delta t$ and better aligns with the non-stationary nature of real-world temporal data. It also offers theoretical justification for the kernel-weighted attention strategy employed in our model.

\section{Appendix C: Dataset Description and Statistics}
\label{appendix:datasets}

This appendix provides full descriptions of the datasets used in our evaluation. 

\subsection{Dynamic Link Prediction} 
For this task, experiments are conducted on eight datasets. Five from TGB \cite{TGB} and three are from JODIE \cite{kumar2019predictingjodie}. Wikipedia dataset is common in both TGB and JODIE. Description of these datasets is as follows:

\noindent \texttt{tgbl-wiki:} This dataset contains information on edits made by users to Wikipedia pages. It forms a bipartite graph with users and wiki pages as nodes. Each edge is associated with features derived from the edits, along with corresponding timestamps.

\noindent  \texttt{tgbl-review:} This dataset contains reviews of Amazon products in the electronics category, collected from 1997 to 2018. Here, users and products form nodes. Ratings range from 1 to 5. It forms a weighted bipartite graph.

\noindent  \texttt{tgbl-coin:} This dataset contains cryptocurrency transaction data~\cite{NEURIPS2022_e245189a_chartlist} from April 1st, 2022 to November 1st, 2022. Nodes represent wallet addresses, and edges capture fund transfers over time.

\noindent  \texttt{tgbl-comment:} This dataset is a Reddit interaction graph where users are nodes and comments represent edges. It spans the years 2005 to 2010.

\noindent  \texttt{tgbl-flight:} This dataset represents a daily international flight network, where nodes are airports and edges correspond to scheduled flights annotated with metadata. Node features include information about the airports. The dataset spans the years 2019 to 2022.

The \texttt{tgbl} dataset statistics is provided in Table \ref{tab:tgbl_stats}.

Additionally, we include following four datasets from JODIE \cite{kumar2019predictingjodie}\footnote{Dataset loader is from: \url{https://github.com/pyg-team/pytorch_geometric/blob/master/torch_geometric/datasets/jodie.py}.} for link prediction. Note that \texttt{tgbl-wiki} and \texttt{jodie-wiki} are the same datasets. Description of the remaining datasets is as follows:

\noindent \texttt{jodie-reddit:} This public dataset consists of one month of user posts made across various subreddits.

\noindent \texttt{jodie-mooc:} This public dataset consists of student actions on a MOOC online course, such as viewing videos, submitting answers, and other interactions.

\noindent \texttt{jodie-lastfm:} This public dataset contains one month of user listening activity, indicating which songs were listened to by which users.

Table \ref{tab:jodie-stats} summarizes the statistics of \texttt{jodie-datasets}.

\subsection{Dynamic Node Classification} The experiments for this task are conducted on five datasets. Four of these are from \cite{TGB}, namely, \texttt{tgbn-trade}, \texttt{tgbn-genre}, \texttt{tgbn-reddit} and \texttt{tgbn-token}. Apart from TGB, we include \texttt{DGraphFin} dataset \cite{DGraphFin}, which contains financial scenarios. Description of these datasets is as follows:

\noindent \texttt{tgbn-trade:} This dataset represents the international agricultural trade network among UN nations from 1986 to 2016, where nodes are countries and edges indicate the annual total trade value of agricultural products.

\noindent \texttt{tgbn-genre:} This is a weighted bipartite interaction network between users and music genres, where edges indicate that a user listened to a genre at a specific time. Edge weights reflect the percentage association of a song with the given genre.

\noindent \texttt{tgbn-reddit:} This is a user–subreddit interaction network where both users and subreddits are nodes. Edges represent user posts on subreddits at specific times, spanning from 2005 to 2019. The task is to predict a user’s interaction frequency with subreddits in the following week.

\noindent \texttt{tgbn-token:} This is a user–token transaction network where both users and cryptocurrency tokens are nodes. Edges represent token transfers from users, with weights indicating the transaction amount. Due to weight disparity, edge weights are log-normalized. The task is to predict a user’s interaction frequency with different tokens in the next week.

\noindent \texttt{dgraph-fin:} A directed, unweighted dynamic graph representing a social network among users of Finvolution Group, where each node is a user and edges denote emergency contact relationships. Edges carry encrypted timestamps indicating when the contact was added. The task is to detect fraudulent users (Class~1) versus normal users (Class~0), using node features and temporal graph structure. 


Statistics for the \texttt{tgbn} datasets are presented in Table~\ref{tab:tgbn_stats}, while those for \texttt{DGraphFin} are shown in Table~\ref{tab:dgraphfin-stats}.

\subsection{Temporal Drift in Inter-Arrival Distributions.} 
Figure~\ref{fig:12_dataset_dist} presents the inter-arrival time distributions $p(\Delta t)$ across 12 temporal graph datasets used in our experiments. We evaluate five link prediction datasets from the TGB benchmark (\texttt{tgbl-wiki}, \texttt{tgbl-review}, \texttt{tgbl-coin}, \texttt{tgbl-comment}, and \texttt{tgbl-flight}) and four node classification datasets (\texttt{tgbn-trade}, \texttt{tgbn-genre}, \texttt{tgbn-token}, and \texttt{tgbn-reddit}) from the same benchmark. To further assess the generalization capability of our method KEAT, we additionally consider four datasets from the JODIE paper—\texttt{jodie-wiki}, \texttt{jodie-mooc}, \texttt{jodie-reddit}, and \texttt{jodie-lastfm}—for link prediction, using the original evaluation setup defined in JODIE. Notably, \texttt{jodie-wiki} and \texttt{tgbl-wiki} represents the same dataset. For each dataset, we compare the inter-arrival distribution $p(\Delta t)$ between the training and validation sets. These visualizations reveal frequent distribution shifts, with validation distributions often exhibiting heavier tails or shifted modes. This motivates our use of temporally robust mechanisms, such as kernel-based attention, which mitigate sensitivity to higher-order temporal statistics that are not consistent across splits.

\subsection{Spectral Density Plots}
The figure \ref{fig:spectral_density} represents the spectral density plots for different datasets on which our experiments were performed. 
Let \( P_f \) denote the probability associated with each frequency component of the signal at the node level, and  
\( M_f \) represent the magnitude of the corresponding frequency component in the fast Fourier Transform (FFT) spectrum at the node level.
Then the formulation for the entropy is given as $H(P) = -\sum_{f} P_f \log P_f $.

We arrive at there values of spectral entropy on a node level and plot the kernel density estimation (KDE) plot for all nodes of each dataset. We first take the timestamps of interaction on a node wise level for a particular dataset under consideration. Once the series of timestamps are noted, the corresponding FFT is taken at a node wise level. These FFTs represent magnitudes/ or intensities of different signals. Here the $P_f$ for a particular signal is calculated as $P_f = M_f/\sum(M_f)$. Once this $P_f$ is calculated for each signal for each node we calculate spectral entropy for each node as mentioned in the expression before. Ideally a purely periodic node would have only one signal leading to 0 spectral entropy. Hence, lower is the value of spectral entropy greater is the degree of periodicity.

Lower entropy reflects greater periodicity in the data, whereas higher entropy denotes increased randomness. Based on the spectral density plots in Figure\ref{fig:spectral_density}, the datasets can be ranked by their level of periodicity, from most to least periodic, as follows: \texttt{tgbl-review}, \texttt{tgbl-coin}, \texttt{tgbl-wiki}, \texttt{tgbl-comment}, and \texttt{tgbl-flight}. This observation shows that many datasets do not display strong periodic patterns. Therefore, relying only on time encodings, which primarily aim to capture periodic behavior, may not be adequate. This motivates the design of KEAT, which enables effective temporal modeling by leveraging edge-aware attention to capture both periodic and non-periodic temporal dynamics.

\begin{figure}[t]
    \centering
    \includegraphics[width=0.98\linewidth]{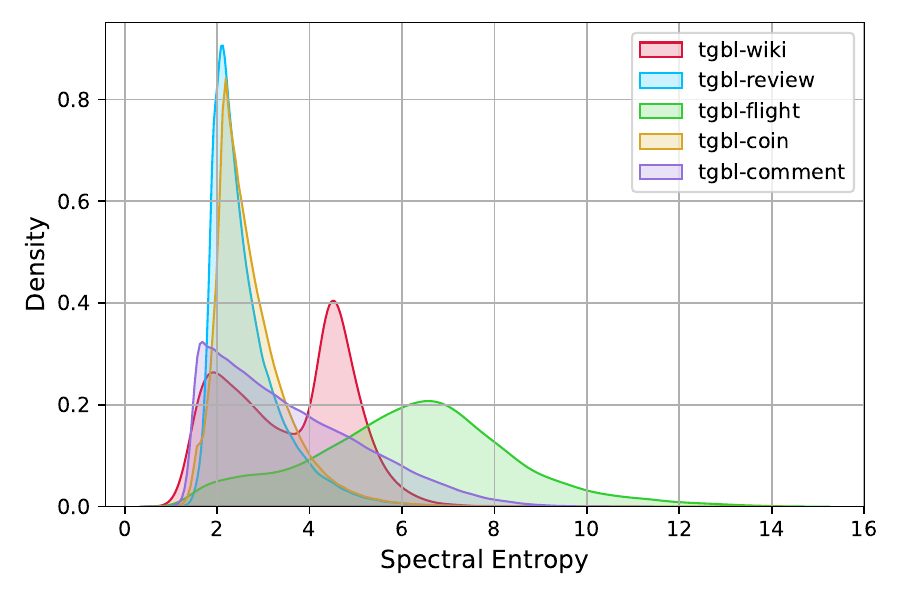}
    \caption{Spectral density for \texttt{tgbl-datasets} illustrating nature of periodicity.}
    \label{fig:spectral_density}
\end{figure}

\section{Appendix D: DyGFormer}

In DyGFormer~\cite{dygformer2023}, for each interaction $(i, j, t_{ij})$, the temporal neighborhoods $\mathcal{N}(i)$ and $\mathcal{N}(j)$ of length $L$ each, are divided into $\frac{L}{L_p}$ non-overlapping patches, each containing $L_p (< L )$ interactions. For each patch $\mathcal{P} = \{(u_k, v_k, t_k)\}{k=1}^{L_p}$, node features, edge features, and time encodings across the patch are concatenated as $\mathbf{h}{\text{flat}} \in \mathbb{R}^{L_p. d}$, $\mathbf{e}{\text{flat}} \in \mathbb{R}^{L_p . d_e}$, and $\boldsymbol{\phi}{\text{flat}} \in \mathbb{R}^{L_p. d_t}$. These are projected independently as :
\[
\tilde{\mathbf{h}} = \mathbf{W}_h \mathbf{h}_{\text{flat}}, \quad
\tilde{\mathbf{e}} = \mathbf{W}_e \mathbf{e}_{\text{flat}}, \quad
\tilde{\boldsymbol{\phi}} = \mathbf{W}_t \boldsymbol{\phi}_{\text{flat}},
\]
where $\mathbf{W}_h \in \mathbb{R}^{d' \times L_p. d}$, $\mathbf{W}_e \in \mathbb{R}^{d' \times L_p. d_e}$, and $\mathbf{W}_t \in \mathbb{R}^{d' \times L_p. d_t}$. The final representation for each patch is
\[
\mathbf{z}_{\mathcal{P}} = [\tilde{\mathbf{h}} \| \tilde{\mathbf{e}} \| \tilde{\boldsymbol{\phi}}] \in \mathbb{R}^{3d'}.
\]

To incorporate KEAT, for each patch we define a representative patch timestamp $t_{\text{patch}} = \frac{1}{L_p} \sum_{k=1}^{L_p} t_k$, and scale the query and key vectors as:
\[
\mathbf{q} = \exp(t_{\text{patch}}) \cdot \mathbf{W}_q \mathbf{z}_{\mathcal{P}}, \quad
\mathbf{k} = \exp(-t_{\text{patch}}) \cdot \mathbf{W}_k \mathbf{z}_{\mathcal{P}},
\]
leading to an attention score where the edge features are modulated by a temporal kernel $\exp(t_{\text{patch}}^{(p)} - t_{\text{patch}}^{(q)})$.
This introduces a directional temporal bias by asymmetrically scaling the queries and keys based on patch timestamps. As a result, the attention mechanism becomes sensitive to temporal dynamics without modifying DyGFormer's architecture.

\section{Appendix E: Impact of Temporal Kernels on Higher-Order Moments} 

In TGNNs, time encodings such as \( f(t) = \sum_{n=0}^\infty c_n t^n \) are often used to capture fine-grained timing patterns. However, higher-order moments \( \mathbb{E}[t^n] \) can become unstable under temporal distribution shift (e.g., when test timestamps differ from training), especially in long-tailed distributions (see Figure~\ref{fig:12_dataset_dist}). To mitigate this, we introduce a temporal kernel \( K(t_i, t_j) = \exp\left(-\frac{|t_i - t_j|}{C}\right) \) that effectively attenuates large \( t \), suppressing higher-order moments in time encoding functions.

\begin{theorem}
Let \( p(t) \) be a probability density function supported on \( [0, \infty) \) such that \( \mathbb{E}[t^n] < \infty \) and \( \mathbb{E}[\psi(t) t^n] < \infty \) for all \( n \geq 0 \), where \( \psi(t) \) is a non-negative, monotonically decreasing kernel function. Let \( \phi(t) = \sum_{n=0}^{\infty} c_n t^n \) be an analytic time encoding function with coefficients \( \{c_n\} \). Define the kernel-to-base ratio for each moment order \( n \) as:
\[
R_n = \frac{\mathbb{E}[\psi(t) t^n]}{\mathbb{E}[t^n]}, \quad \text{with } R_0 = \mathbb{E}[\psi(t)].
\]
Then, the sequence \( \{R_n\}_{n=0}^\infty \) is strictly decreasing and converges to zero: \( \lim_{n \to \infty} R_n = 0 \). As a result, the expected kernel-weighted time encoding \( \mathbb{E}[\psi(t) \cdot \phi(t)] \) becomes increasingly dominated by lower-order terms \( \phi \). This formulation naturally extends to expectations of the form \( \mathbb{E}[\psi(t) \cdot \bar{\mathbf{e}}_{ij}] \), as used in Equations \ref{eq:modulated_attention} and \ref{eq:modulated_update}.
\end{theorem}

\paragraph{\textbf{Proof.}}
To show monotonicity, consider the ratio:
\[
\frac{R_{n+1}}{R_n} = \frac{\mathbb{E}[t^{n+1} \psi(t)] \cdot \mathbb{E}[t^n]}{\mathbb{E}[t^n \psi(t)] \cdot \mathbb{E}[t^{n+1}]}.
\]
Both numerator and denominator are expectations involving \( t^k \) under two different distributions: one weighted by \( \psi(t) \) and one unweighted. Since \( \psi(t) \) is a non-negative, strictly decreasing function, it downweights larger values of \( t \), leading to:
\[
\frac{\mathbb{E}[t^{n+1} \psi(t)]}{\mathbb{E}[t^n \psi(t)]} < \frac{\mathbb{E}[t^{n+1}]}{\mathbb{E}[t^n]}.
\]
Hence, \( R_{n+1} < R_n \).

To prove convergence to zero, observe that \( \psi(t) \leq \psi(0) \) and decays faster than any polynomial growth in \( t \). Thus:
\[
\lim_{n \to \infty} \frac{\mathbb{E}[t^n \psi(t)]}{\mathbb{E}[t^n]} = 0,
\]
as the numerator decays faster than the denominator. Therefore, \( \lim_{n \to \infty} R_n = 0 \).

This implies that for the analytic expansion \( \phi(t) = \sum c_n t^n \), we have:
\[
\mathbb{E}[\psi(t) \cdot \phi(t)] = \sum_{n=0}^\infty c_n \mathbb{E}[\psi(t) t^n] = \sum_{n=0}^\infty c_n R_n \mathbb{E}[t^n],
\]
where the coefficients \( R_n \) decay rapidly, suppressing higher-order moments.

A visual illustration of higher-order moment suppression is shown in Figure \ref{fig:moment_decay}.

\begin{figure}[t]
    \centering
    \includegraphics[width=0.98\linewidth]{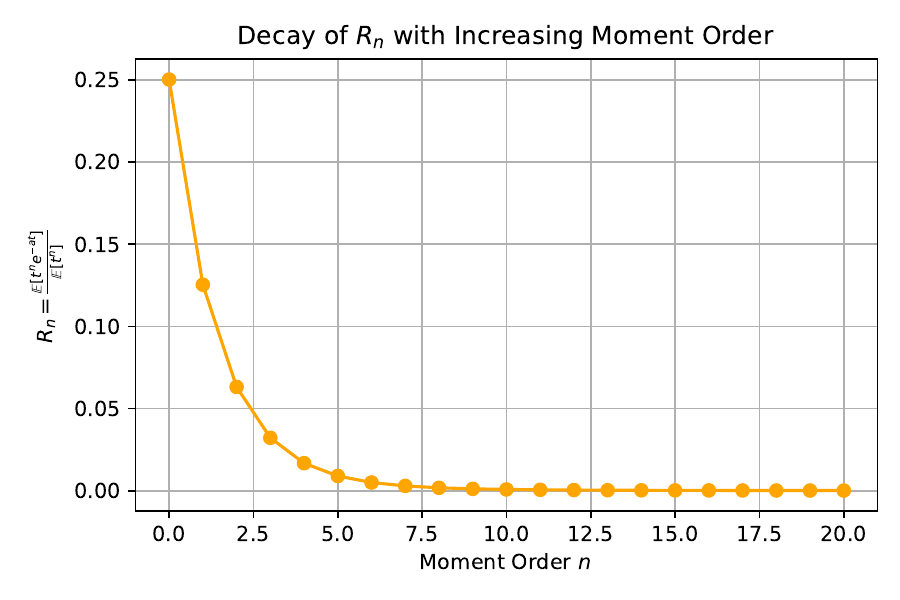}
    \caption{Decay of the ratio $R_n = \frac{\mathbb{E}[t^n e^{-a t}]}{\mathbb{E}[t^n]}$ for increasing moment order $n$, where $t \sim \text{Exp}(1)$. This demonstrates how higher-order moments are increasingly suppressed under exponential kernel weighting.}
    \label{fig:moment_decay}
\end{figure}

\subsection{Variance Reduction in KEAT}

\begin{theorem}[Variance Reduction via Temporal Kernel]
In the KEAT formulation, the attention logits for node \( i \) with neighbor \( j \) take the form:
\[
s_{ij}^{(K)} := \mathbf{q}_i^\top \left( \mathbf{k}_j + \psi(\Delta t_{ij}) \cdot \mathbf{k}^{(\mathrm{edge})}_{ij} \right)
= X_j + \psi_{ij} Y_j,
\]
where:
\begin{itemize}
    \item \( X_j := \mathbf{q}_i^\top \mathbf{k}_j \) is the node-to-node interaction term,
    \item \( Y_j := \mathbf{q}_i^\top \mathbf{k}^{(\mathrm{edge})}_{ij} \) captures edge-specific contributions,
    \item \( \psi(\cdot) \) is a temporal kernel (e.g., exponential decay), and \( \psi_{ij} := \psi(\Delta t_{ij}) \in (0, 1] \) is its value for the edge from \( j \) to \( i \).
\end{itemize}

Assume \( \{(X_j, Y_j)\} \) are i.i.d.\ with variances \( \sigma_X^2 \), \( \sigma_Y^2 \), and correlation \( \rho \in [-1, 1] \).  
Then the variance of the kernel-modulated logits satisfies:
\[
\mathrm{Var}[s_{ij}^{(K)}] \leq \mathrm{Var}[s_{ij}^{(0)}] \quad
\text{whenever} \quad \sigma_Y (1 + \psi_{ij}) \geq 2 \sigma_X.
\]

Thus, the temporal kernel suppresses variance from edge-based noise, yielding more stable attention logits.
\end{theorem}

\paragraph{\textbf{Proof.}}
Let the kernel-modulated and unmodulated logits be:
\[
s_{ij}^{(K)} = X_j + \psi_{ij} Y_j, \quad s_{ij}^{(0)} = X_j + Y_j,
\]
where \( \psi_{ij} := \psi(\Delta t_{ij}) \in (0, 1] \).

We compute the variances under the assumption that \( (X_j, Y_j) \) are i.i.d.\ with variances \( \sigma_X^2, \sigma_Y^2 \), and correlation \( \rho \in [-1, 1] \):

\begin{align*}
\mathrm{Var}[s_{ij}^{(K)}] 
    &= \mathrm{Var}[X_j + \psi_{ij} Y_j] \\
    &= \sigma_X^2 + \psi_{ij}^2 \sigma_Y^2 + 2 \psi_{ij} \rho \sigma_X \sigma_Y \\[10pt]
\mathrm{Var}[s_{ij}^{(0)}] 
    &= \sigma_X^2 + \sigma_Y^2 + 2 \rho \sigma_X \sigma_Y
\end{align*}

Let \( \Delta := \mathrm{Var}[s_{ij}^{(0)}] - \mathrm{Var}[s_{ij}^{(K)}] \). Then:

\begin{align*}
\Delta &= (\sigma_X^2 + \sigma_Y^2 + 2\rho \sigma_X \sigma_Y) \\
&\quad - (\sigma_X^2 + \psi_{ij}^2 \sigma_Y^2 + 2 \psi_{ij} \rho \sigma_X \sigma_Y) \\
&= \sigma_Y^2 (1 - \psi_{ij}^2) + 2 \rho \sigma_X \sigma_Y (1 - \psi_{ij}).
\end{align*}

Factoring:
\[
\Delta = (1 - \psi_{ij}) \left[ \sigma_Y^2 (1 + \psi_{ij}) + 2 \rho \sigma_X \sigma_Y \right].
\]

Define the bracketed expression as:
\[
B := \sigma_Y^2 (1 + \psi_{ij}) + 2 \rho \sigma_X \sigma_Y.
\]

Since \( 1 - \psi_{ij} > 0 \), the sign of \( \Delta \) depends on \( B \). In the worst case (\( \rho = -1 \)), we have:
\[
B_{\min} = \sigma_Y^2 (1 + \psi_{ij}) - 2 \sigma_X \sigma_Y.
\]

Thus, \( \Delta \geq 0 \) (i.e., variance is reduced) whenever:
\[
\sigma_Y (1 + \psi_{ij}) \geq 2 \sigma_X.
\]

This completes the proof.

\subsection{Extension: Correlated Neighbor Terms}

\noindent
Even when the i.i.d.\ assumption on neighbor pairs \( (X_j, Y_j) \) is relaxed, we can still show that incorporating the temporal kernel leads to a reduction in the variance of the aggregated attention logits under mild conditions. In realistic temporal graphs, interactions from neighbors may exhibit correlations due to structural proximity or temporal locality. The following extension demonstrates that variance reduction continues to hold so long as these correlations are non-negative.

Let the attention logits in KEAT be given by
\[
s_{ij}^{(K)} := X_j + \psi_{ij} Y_j,
\]
where \( \psi_{ij} := \psi(\Delta t_{ij}) \in (0,1] \), and define the neighborhood-averaged logit:
\[
\bar{s}_i^{(K)} := \frac{1}{|\mathcal{N}(i)|} \sum_{j \in \mathcal{N}(i)} s_{ij}^{(K)}.
\]

Assume that for all \( j \neq l \), the cross-covariances satisfy:
\[
\mathrm{Cov}[Y_j, Y_l] \geq 0, \quad \mathrm{Cov}[X_j, Y_l] \geq 0, \quad \mathrm{Cov}[Y_j, X_l] \geq 0.
\]

Then,
\[
\mathrm{Var}[\bar{s}_i^{(K)}] \leq \mathrm{Var}[\bar{s}_i^{(0)}].
\]

\paragraph{\textbf{Proof.}}
Let \( s_{ij}^{(0)} := X_j + Y_j \). The variance of the neighborhood-averaged logit is:
\[
\mathrm{Var}[\bar{s}_i^{(K)}] = \frac{1}{n^2} \sum_{j,l} \mathrm{Cov}[X_j + \psi_{ij} Y_j,\ X_l + \psi_{il} Y_l].
\]

Taking the difference from the unmodulated case and expanding yields:

\begin{align*}
\Delta &:= \mathrm{Var}[\bar{s}_i^{(0)}] - \mathrm{Var}[\bar{s}_i^{(K)}] \\
&= \frac{1}{n^2}\sum_{j,l}\bigl[(1 - \psi_{ij}\psi_{il})\mathrm{Cov}[Y_j,Y_l] \\
&\quad + (1 - \psi_{ij})\mathrm{Cov}[X_j,Y_l] \\
&\quad + (1 - \psi_{il})\mathrm{Cov}[Y_j,X_l]\bigr].
\end{align*}

Each term in the sum is non-negative under the stated covariance conditions and \( \psi_{ij} \in (0,1] \), implying \( \Delta \geq 0 \).

\paragraph{Remark.}
This generalization shows that the variance reduction effect of KEAT holds even in the presence of correlated or temporally clustered neighbors, so long as the correlations are non-negative. This setting captures realistic temporal graph phenomena like structurally similar interactions. The temporal kernel acts as a soft attention prior that dampens noisy or high-variance edge contributions, improving the stability of attention logits under broader assumptions.

\begin{figure}[t]
\centering
\begin{tikzpicture}
\begin{axis}[
    width=0.98\linewidth,
    height=6cm,
    xlabel={Time difference $\Delta t$},
    ylabel={Kernel weight},
    legend style={at={(0.95,0.95)}, anchor=north east, font=\small},
    grid=major,
    xmin=0, xmax=10,
    ymin=0, ymax=1.05,
    thick,
    samples=200,
    domain=0:10,
    axis lines=left
]

\addplot[red, dashed, thick] {exp(-x / 2)};
\addlegendentry{Laplacian: $e^{-\Delta t / \sigma}$}

\addplot[blue, very thick] {exp(-x^2 / 4)};
\addlegendentry{RBF: $e^{-\Delta t^2 / \sigma^2}$}

\addplot[green!60!black, dotted, thick] {0.8*exp(-0.5*x) + 0.2*sin(deg(x))/2};
\addlegendentry{MLP: learned decay $f(\Delta t; \theta)$}

\end{axis}
\end{tikzpicture}
\caption{Comparison of temporal kernels used in KEAT for edge modulation. Laplacian and RBF apply fixed exponential decays, while the MLP-based kernel learns a flexible decay pattern from data, allowing richer temporal biasing.}
\label{fig:kernels}
\end{figure}
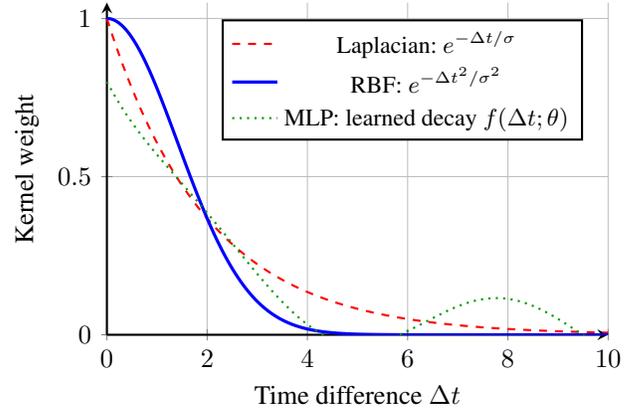

\subsection{Different Types of Kernels}
Figure~\ref{fig:kernels} illustrates the behavior of three temporal kernels used in KEAT to modulate edge time interactions. The Laplacian and RBF kernels apply fixed decay functions. The Laplacian kernel exhibits an exponential drop with a linear rate of decay, while the RBF kernel shows a sharper and symmetric decay. These fixed forms are effective for modeling simple recency-based interactions. In contrast, the MLP-based kernel learns its decay pattern directly from data, allowing it to capture complex and potentially non-monotonic temporal dependencies. This learnable flexibility enables KEAT to better adapt to diverse temporal structures in dynamic graphs.

\section{Appendix F: Baselines}

\begin{table*}[ht]
\centering
\small
\rowcolors{2}{white}{gray!5}
\begin{tabular}{lccccccc}
\toprule
\textbf{Method} & \texttt{wiki} & \texttt{review} & \texttt{coin} & \texttt{comment} & \texttt{flight} & \textbf{Avg Rank $\downarrow$} & \textbf{Ranking MRR $\uparrow$} \\
\midrule
TNCN & 3 & 3 & 1 & 1 & 1 & 1.8 & 0.733 \\
TGN & 10 & 5 & 4 & 4 & 2 & 5.0 & 0.260 \\
CTAN & 5 & 2 & 3 & 2 & \cellcolor{red!20}17 & 5.8 & 0.318 \\
DyGFormer & 1 & 7 & 2 & 3 & \cellcolor{red!20}17 & 6.0 & 0.407 \\
$\text{EdgeBank}_{\text{tw}}$ & 6 & 15 & 5 & 6 & 4 & 7.2 & 0.170 \\
DyRep & 16 & 8 & 6 & 5 & 3 & 7.6 & 0.178 \\
$\text{EdgeBank}_\ensuremath{\infty}$ & 7 & 16 & 7 & 7 & 5 & 8.4 & 0.138 \\
NAT & 2 & 6 & \cellcolor{red!20}17 & \cellcolor{red!20}17 & \cellcolor{red!20}17 & 11.8 & 0.169 \\
CAWN & 4 & 10 & \cellcolor{red!20}17 & \cellcolor{red!20}17 & \cellcolor{red!20}17 & 13.0 & 0.105 \\
GraphMixer & 15 & 1 & \cellcolor{red!20}17 & \cellcolor{red!20}17 & \cellcolor{red!20}17 & 13.4 & 0.249 \\
EGCNo (UTG) & 9 & 9 & \cellcolor{red!20}17 & \cellcolor{red!20}17 & \cellcolor{red!20}17 & 13.8 & 0.080 \\
TGAT & 14 & 4 & \cellcolor{red!20}17 & \cellcolor{red!20}17 & \cellcolor{red!20}17 & 13.8 & 0.100 \\
HTGN (UTG) & 8 & 13 & \cellcolor{red!20}17 & \cellcolor{red!20}17 & \cellcolor{red!20}17 & 14.4 & 0.076 \\
GCN (UTG) & 12 & 12 & \cellcolor{red!20}17 & \cellcolor{red!20}17 & \cellcolor{red!20}17 & 15.0 & 0.069 \\
TCL & 13 & 11 & \cellcolor{red!20}17 & \cellcolor{red!20}17 & \cellcolor{red!20}17 & 15.0 & 0.069 \\
GCLSTM (UTG) & 11 & 14 & \cellcolor{red!20}17 & \cellcolor{red!20}17 & \cellcolor{red!20}17 & 15.2 & 0.068 \\
\bottomrule
\end{tabular}
\caption{Ranking of link prediction methods on \texttt{TGB} benchmark datasets. Missing scores are penalized with rank 17 and highlighted in red. Numbers are taken from TGB link prediction leaderboard \cite{TGB} as on 4th Aug 2025.}
\label{tab:ranking_penalty}
\end{table*}

Table \ref{tab:ranking_penalty} describes the ranks for each baseline on different \texttt{tgbl} link prediction datasets. We select TNCN, TGN, CTAN, and DyGFormer as our primary baselines due to their consistently strong performance across multiple datasets in the TGB benchmark\footnote{\url{https://tgb.complexdatalab.com/docs/leader_linkprop/}}. These models represent a diverse set of temporal graph architectures and provide competitive results in both average ranking and ranking MRR. In contrast, several other methods exhibit poor generalization, scalability (highlighted by red cells in Table \ref{tab:ranking_penalty}) or lack robustness across datasets, often missing scores and receiving penalized ranks. To maintain clarity and focus in our evaluations, we omit these underperforming models from detailed analysis.

\section{Appendix G: Additional Experimental Details and Results}

\subsection{Implementation Details}
We employ the Adam optimizer with a learning rate of $1 \times 10^{-4}$, a batch size of 200, and dimensionality of 100 for embeddings, memory, and time encodings. We apply early stopping with a tolerance of $1 \times 10^{-6}$ and a patience of 5 epochs. For the \textsc{Wiki} dataset, we increase the number of training epochs to 200 and set the early stopping patience to 25, owing to its smaller size. All models are trained over 5 random seeds (1 through 5), and we report the mean and standard deviation of the evaluation metrics. All experiments are conducted on an Ubuntu system equipped with an NVIDIA A100 80GB GPU and 128\,GB of RAM.

In KEAT-TGN, the official implementation of TGN\footnote{\url{https://github.com/shenyangHuang/TGB/tree/main}} in TGB \cite{TGB} is used. For DyGFormer, we use the official implementation provided by the authors\footnote{\url{https://github.com/yule-BUAA/DyGLib_TGB}}.
For the {JODIE} datasets, we use the implementation available in \texttt{PyTorch Geometric}\footnote{\url{https://github.com/pyg-team/pytorch_geometric/blob/master/torch_geometric/datasets/jodie.py}}. For the DGraphFin dataset, we adopt the publicly available implementation of GEARSage\footnote{\url{https://github.com/storyandwine/GEARSage-DGraphFin}}.

\paragraph{Hyperparameters.} As KEAT is designed to be lightweight and plug-and-play, we retain the exact hyperparameters used by the respective baseline models without any additional tuning. This ensures a fair comparison and highlights the compatibility of KEAT with existing architectures. The only additional hyperparameter introduced by KEAT is the training set standard deviation ($\sigma$), which is used for kernel modulation and is directly computed and reported in the code for reproducibility.

\begin{table}[t]
\centering
\small

\begin{tabular}{cccc}
\toprule
\textbf{Node Mod.} & \textbf{Edge Mod.} & \textbf{Val MRR} & \textbf{Test MRR} \\
\midrule
\cmark & \xmark  & 0.374 \scriptsize{$\pm$ 0.163} & 0.344 \scriptsize{$\pm$ 0.169} \\
\cmark & \cmark & 0.364 \scriptsize{$\pm$ 0.021} & 0.360 \scriptsize{$\pm$ 0.024} \\
\xmark  & \cmark & \textbf{0.515} \scriptsize{$\pm$ 0.053} & \textbf{0.474} \scriptsize{$\pm$ 0.031} \\
\xmark  & \xmark  & 0.411 \scriptsize{$\pm$ 0.074} & 0.391 \scriptsize{$\pm$ 0.064} \\
\bottomrule
\end{tabular}

\caption{Ablation of kernel modulation on nodes and edges. Results are reported as MRR (mean $\pm$ std) on the \texttt{tgbl-wiki} dataset.}
\label{tab:node_edge_modulation}
\end{table}

\subsection{Effect of Node and Edge Modulation}
We analyze the contribution of kernel-based modulation on node vs. edge representations (Table~\ref{tab:node_edge_modulation}, \texttt{tgbl-wiki}). Edge modulation alone yields the best performance: test MRR of \textbf{0.474} {\scriptsize$\pm$ 0.031}. In contrast, modulating only node features results in lower performance and higher variance. Combining both is better than node-only, but still inferior to edge-only. These results suggest that temporal modulation at the edge level is more effective, likely because edges directly influence message passing and modulate attention scores.

\begin{figure}[ht]
    \centering
    \includegraphics[width=0.9\linewidth]{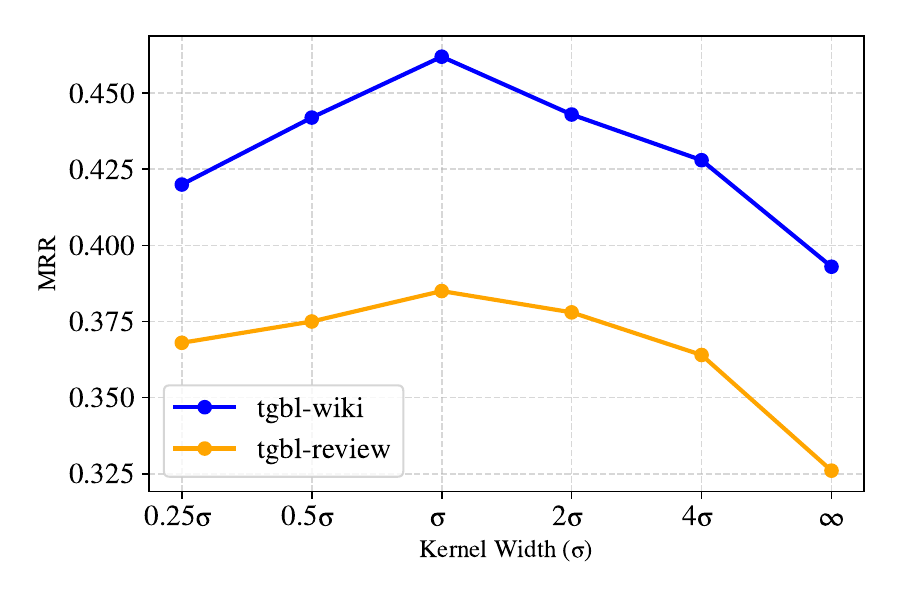}
    \caption{Kernel Width $\sigma$ sensitivity for \texttt{tgbl-wiki}}
    \label{fig:wiki_kernel_width}
\end{figure}

\subsection{Kernel Width Sensitivity}
In Figure \ref{fig:wiki_kernel_width}, we analyze the robustness of KEAT under the Laplacian kernel by varying the kernel width $\lambda$ relative to the standard deviation $\sigma$ of time differences in the training set. As shown in the table, the mean reciprocal rank (MRR) improves as $\lambda$ increases, peaking when $\lambda = \sigma$ with values of 0.474 on \texttt{wiki} and 0.385 on \texttt{review}. Both smaller ($0.25\sigma$, $0.5\sigma$) and larger ($2\sigma$, $4\sigma$, $\infty$) values lead to lower MRR. We have used $\sigma > 5000$ for $\infty$ in the code. This suggests that setting $\lambda$ close to $\sigma$ provides an effective balance between emphasizing recent interactions and retaining longer-term dependencies, showcasing KEAT’s robustness to kernel width choice.

\subsection{Understanding Ranking in MRR}
\begin{figure}[ht]
    \centering
    \includegraphics[width=0.9\linewidth]{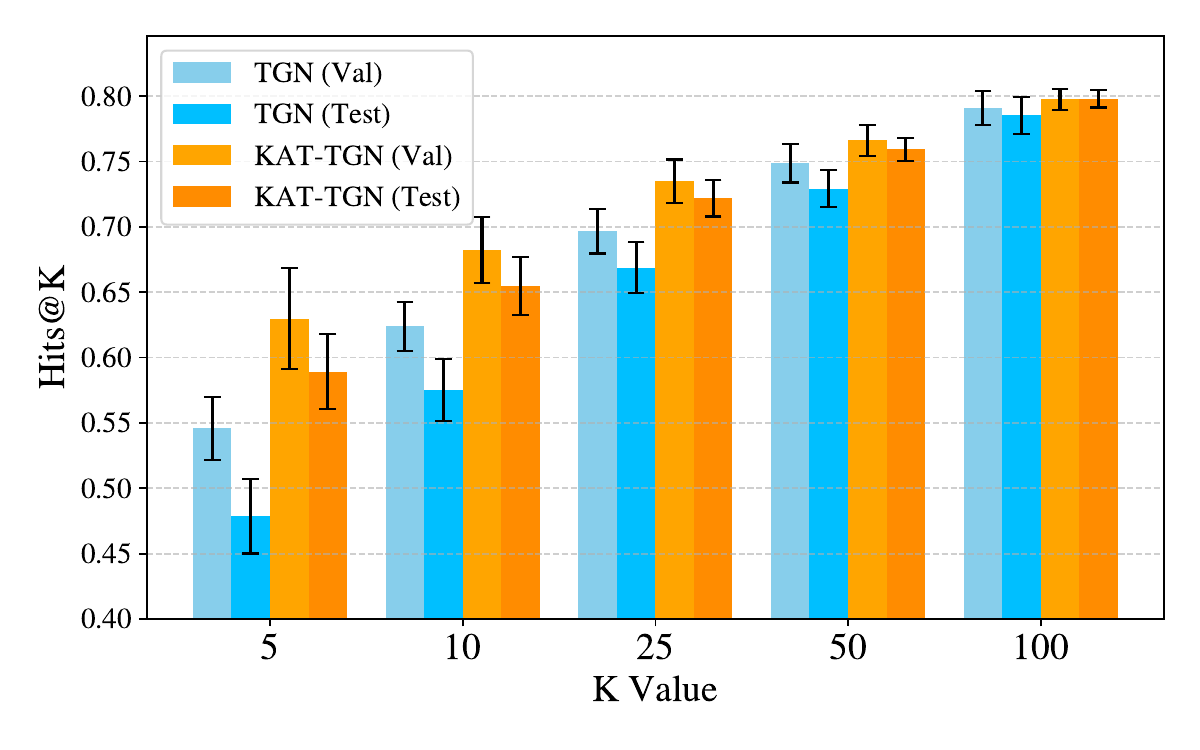}
    \caption{Hits@K for \texttt{tgbl-wiki}}
    \label{fig:wiki_hits}
\end{figure}

To better interpret the ranking of predicted links beyond the MRR metric, we evaluate performance using Hits@K on the \texttt{tgbl-wiki} dataset for various cutoff values of $K$. As shown in Figure~\ref{fig:wiki_hits}, our model \textbf{KEAT-TGN} consistently outperforms the baseline \textbf{TGN} across all $K$ values, both on the validation and test sets. For instance, on the test set, KEAT-TGN achieves a Hits@10 of 65.46\%, compared to 57.51\% for TGN. This consistent improvement demonstrates that KEAT not only improves average ranking (as captured by MRR), but also ranks more correct links among the top-$K$ predictions, enhancing practical utility in retrieval settings.

\subsection{Results on \texttt{tgbl-flight} (Link Prediction)}

\begin{table}[t]
    \centering
    \small
   
    \begin{tabular}{lcc}
    \toprule
    \textbf{Method} & \textbf{Split} & \texttt{tgbl-flight} \\
    \midrule

    \multirow{2}{*}{$\text{EdgeBank}_{\text{tw}}$} 
    & Val  & 0.363 \\
    & Test & 0.387 \\

    \midrule

    \multirow{2}{*}{$\text{EdgeBank}_\ensuremath{\infty}$} 
    & Val  & 0.166 \\
    & Test & 0.167 \\

    \midrule

    \multirow{2}{*}{DyRep} 
    & Val  & 0.573 \scriptsize{$\pm$ 0.013} \\
    & Test & 0.556 \scriptsize{$\pm$ 0.014} \\

    \midrule
    
    \multirow{2}{*}{TGN} & Val & 0.731 \scriptsize{$\pm$ 0.010} \\
                     & Test & 0.705 \scriptsize{$\pm$ 0.020} \\

    \cmidrule{2-3}
    \multirow{2}{*}{KEAT-TGN} & Val & 0.747 \scriptsize{$\pm$ 0.013} \\
                                  & Test & \textbf{0.738} \scriptsize{$\pm$ 0.022} \\

    \cmidrule{2-3}
    
    \multirow{1}{*}{\shortstack[l]{Improvement}} 
    & Test & +3.30\% \\

    \bottomrule
    \end{tabular}
     \caption{Dynamic link prediction results on \texttt{tgbl-flight} using MRR metric.}
    \label{tab:link_pred_results_tgbl_flight}
    \end{table}

On the \texttt{tgbl-flight} dataset, KEAT-TGN achieves the best performance, improving over the TGN baseline by 3.30\% in MRR on the test set (0.738 vs. 0.705). Refer Table \ref{tab:link_pred_results_tgbl_flight} for results. This demonstrates the effectiveness of KEAT in enhancing temporal modeling without altering the backbone architecture. Compared to earlier methods like DyRep and EdgeBank, KEAT-TGN yields significantly higher accuracy. Additionally, while DyGFormer results are not reported here due to severe scalability issues, requiring over 24 hours to complete a single training epoch, KEAT-TGN remains efficient and practical for large-scale temporal graphs.

\begin{table}[t]
\centering
\resizebox{0.9\columnwidth}{!}{%
\begin{tabular}{lccc}
\toprule
\textbf{Dataset} & \textbf{Model} & \textbf{Val} & \textbf{Test} \\
\midrule
\multirow{3}{*}{\texttt{tgbn-trade}} 
  & TGN      & 0.395 {\scriptsize$\pm$ 0.002} & 0.374 {\scriptsize$\pm$ 0.001} \\
  & KEAT-TGN  & \textbf{0.464} {\scriptsize$\pm$ 0.006} & \textbf{0.427} {\scriptsize$\pm$ 0.010} \\
  & \scriptsize{Gain} & \scriptsize{+6.92\%} & \scriptsize{+5.30\%}\\
\midrule
\multirow{3}{*}{\texttt{tgbn-genre}} 
  & TGN      & 0.403 {\scriptsize$\pm$ 0.010} & 0.367 {\scriptsize$\pm$ 0.058} \\
  & KEAT-TGN  & \textbf{0.420} {\scriptsize$\pm$ 0.003} & \textbf{0.420} {\scriptsize$\pm$ 0.002} \\
  & \scriptsize{Gain} & \scriptsize{+1.73\%} & \scriptsize{+5.29\%}\\
\midrule
\multirow{3}{*}{\texttt{tgbn-reddit}} 
  & TGN      & \textbf{0.379} {\scriptsize$\pm$ 0.004} & 0.315 {\scriptsize$\pm$ 0.020} \\
  & KEAT-TGN  & 0.372 {\scriptsize$\pm$ 0.002} & \textbf{0.317} {\scriptsize$\pm$ 0.001} \\
  & \scriptsize{Gain} & \scriptsize{-0.7\%} & \scriptsize{+0.20\%}\\
  
\midrule
\multirow{3}{*}{\texttt{tgbn-token}} 
  & TGN      & 0.189 {\scriptsize$\pm$ 0.005} & 0.141 {\scriptsize$\pm$ 0.006} \\
  & KEAT-TGN  & \textbf{0.194} {\scriptsize$\pm$ 0.008} & \textbf{0.152} {\scriptsize$\pm$ 0.018} \\
  & \scriptsize{Gain} & \scriptsize{+0.50\%} & \scriptsize{+1.1\%}\\
  
\bottomrule
\end{tabular}
}
\caption{Node classification performance (NDCG@10).}
\label{tab:node_classification_ndcg}
\end{table}

\subsection{Results on \texttt{jodie-dataset} (Link Prediction)}

Table~\ref{tab:jodie_results} presents dynamic link prediction performance on the four \texttt{jodie} datasets. We compare KEAT-TGN with the baseline TGN across multiple metrics, including binary cross-entropy loss, average precision (AP), and area under the ROC curve (AUC), on both validation and test sets. KEAT-TGN consistently improves upon TGN across all datasets and metrics. For example, on \texttt{jodie-wiki}, KEAT-TGN achieves a lower loss (0.3895 vs. 0.4065) and higher test AP (0.9660 vs. 0.9593), demonstrating better predictive performance. The improvements are particularly notable on datasets such as \texttt{jodie-lastfm}, where KEAT-TGN yields a significant gain in test AP (0.7438 vs. 0.6193), highlighting its ability to generalize well even in sparse user-item interaction settings. These results underscore the effectiveness of our kernelized temporal attention mechanism in enhancing the link prediction capability of the base TGN model.

\paragraph{Interpretability and Temporal Precision.}
As discussed in the main paper, KEAT enables time-aware edge modulation that results in temporally precise and context-sensitive attention distributions. In Appendix Figure~\ref{fig:attn_summary_1} and Figure \ref{fig:attn_for_10_neighbors_1}, we visualize attention weights over neighbors across varying $\Delta t$, showing how KEAT produces sharp transitions in focus that align with temporal relevance. This contrasts with standard attention mechanisms that often produce flattened or ambiguous weights. The resulting heatmaps demonstrate KEAT's improved interpretability: the model precisely adjusts its attention as edge timestamps become older, preserving temporal fidelity and offering intuitive explanations of which past interactions are deemed important. This clarity in attention evolution enhances both transparency and analytical utility for downstream applications. To further support this, we include two additional examples with varied attention behaviors (see Figures \ref{fig:attn_summary_ex_2},\ref{fig:n_summary_ex_2} for second example and Figures \ref{fig:attn_summary_ex_3},\ref{fig:n_summary_ex_3} for third example).

\begin{figure}[ht]
    \centering
    \includegraphics[width=0.9\linewidth]{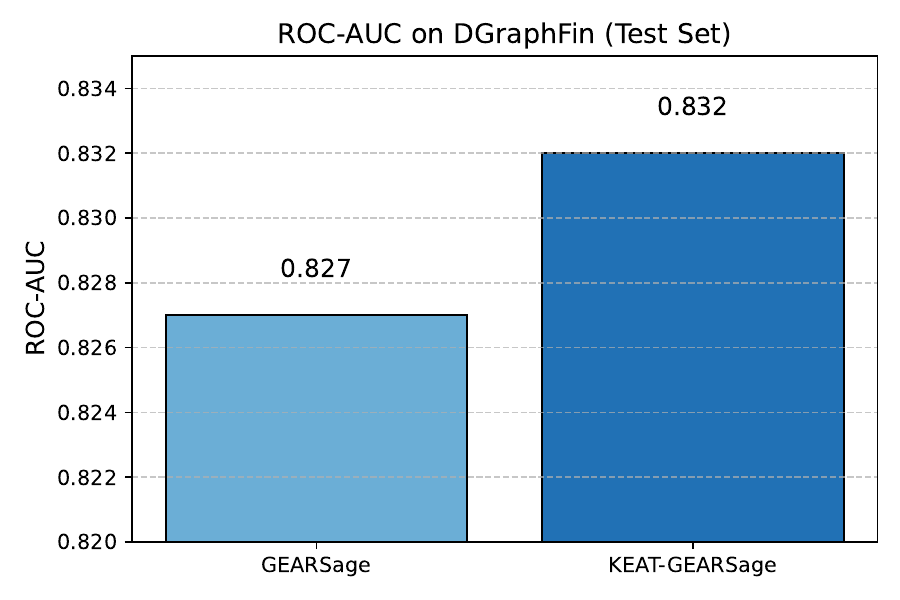}
    \caption{Test ROC-AUC scores on DGraphFin for node classification.}
    \label{fig:dgraphfin_keat}
\end{figure}

\subsection{Results on DGraphFin (Node Classification)}
The ROC-AUC scores on the DGraphFin test set clearly indicate (see Figure \ref{fig:dgraphfin_keat}) that KEAT-GEARSage outperforms the base GEARSage model, achieving a higher score of 0.832 compared to 0.827. This improvement demonstrates KEAT's ability to enhance temporal modeling by better leveraging time-dependent edge information. The gain highlights KEAT’s effectiveness in capturing subtle temporal cues essential for fraud detection in dynamic financial networks.

\paragraph{Time Encoding Ablation.}
We evaluate the effect of KEAT across different time encoding strategies on the \texttt{tgbl-wiki}, \texttt{tgbl-review} and \texttt{tgbl-coin} datasets. As shown in Table~\ref{tab:time_encoding_ablation_2}, incorporating KEAT consistently improves performance across both datasets and models. For instance, GraphMixer with KEAT achieves a test MRR of 0.409 on \texttt{tgbl-review} and 0.714 on \texttt{tgbl-coin}, outperforming the baseline variants without KEAT. Similar trends are observed with TGN/TGAT models, where KEAT yields consistent gains over the default time encodings. Notably, we do not include results for LETe on \texttt{tgbl-review} and \texttt{tgbl-coin} , as it requires significantly more parameters and memory to train as shown in Table \ref{tab:param-count}, which hinders its scalability for these benchmarks.

\subsection{Model Parameters}
Table \ref{tab:param-count} compares trainable parameter counts across various time encoding strategies when used with the TGN backbone. While \texttt{GraphMixer} relies on non-trainable, fixed encodings and thus has slightly fewer parameters, it lacks the flexibility to adapt temporal patterns. The standard \texttt{TGN/TGAT} introduces minimal parameter overhead for time encoding (\texttt{+200}). LeTE-based approaches add more trainable parameters depending on the modeling strategy, with the highest cost for fully learned Fourier or Spline encodings. This demonstrates that TGN allows efficient parameter sharing and modular integration of expressive time encodings, striking a balance between model capacity and learnability.

\section{Appendix H: Code and Reproducibility}
To support reproducibility and fair comparison, we release our complete codebase, including training and evaluation scripts, under an anonymous folder in supplementary zip. The code includes instructions, dataset links, and configuration files to replicate experiments presented in this paper. We ensure that the released code adheres to double-blind review standards and will be made public upon acceptance.

\begin{table*}[t]
\centering

\begin{tabular}{llccccc}
\toprule
\textbf{Dataset} & \textbf{Method} & \textbf{Loss} ↓ & \textbf{Val AP} ↑ & \textbf{Val AUC} ↑ & \textbf{Test AP} ↑ & \textbf{Test AUC} ↑ \\
\midrule
\multirow{2}{*}{\texttt{jodie-wiki}} 
& TGN       & 0.4065 & 0.9686 & 0.9659 & 0.9593 & 0.9575 \\
& KEAT-TGN   & \textbf{0.3895} & \textbf{0.9733} & \textbf{0.9703} & \textbf{0.9660} & \textbf{0.9634} \\
\midrule
\multirow{2}{*}{\texttt{jodie-reddit}} 
& TGN       & {0.6876} & 0.9185 & 0.9147 & 0.9190 & 0.9149 \\
& KEAT-TGN   & \textbf{0.6807} & \textbf{0.9264} & \textbf{0.9219} & \textbf{0.9213} & \textbf{0.9166} \\
\midrule
\multirow{2}{*}{\texttt{jodie-mooc}} 
& TGN       & {0.5550} & 0.8626 & {0.8875} & 0.8477 & {0.8794} \\
& KEAT-TGN   & \textbf{0.5420} & \textbf{0.8771} & \textbf{0.8991} & \textbf{0.8637} & \textbf{0.8898} \\

\midrule
\multirow{2}{*}{\texttt{jodie-lastfm}} 
& TGN       & {0.9760} & 0.7887 & 0.7891 & 0.6193 & 0.6245 \\
& KEAT-TGN   & \textbf{0.9719} & \textbf{0.7968} & \textbf{0.7939} & \textbf{0.7438} & \textbf{0.7420} \\
\bottomrule
\end{tabular}

\caption{Link prediction results on \texttt{jodie-datasets}, using one negative sample per test edge. The model is trained with binary cross-entropy loss, and performance is reported using average precision (AP) and area under the curve (AUC). The implementation follows the \texttt{pytorch-geometric} framework.}

\label{tab:jodie_results}
\end{table*}

\begin{table*}[t]
\centering

\begin{tabular}{lccccccc}
\toprule
\textbf{Time Encoding} & \textbf{Model} & \multicolumn{2}{c}{\texttt{tgbl-wiki}} & \multicolumn{2}{c}{\texttt{tgbl-review}} & \multicolumn{2}{c}{\texttt{tgbl-coin}} \\
\cmidrule{3-8}
 & & Val & Test & Val & Test & Val & Test \\
\midrule
\texttt{GraphMixer} & w/o KEAT & 0.477 \scriptsize{$\pm$ 0.041} & 0.432 \scriptsize{$\pm$ 0.061} & 0.337 {\scriptsize$\pm$ 0.007} & 0.404 {\scriptsize$\pm$ 0.008} & 0.659 {\scriptsize$\pm$ 0.011} & 0.688 {\scriptsize$\pm$ 0.028} \\
\texttt{GraphMixer} & w/ KEAT  & \textbf{0.489} \scriptsize{$\pm$ 0.033} & \textbf{0.458} \scriptsize{$\pm$ 0.025} & \textbf{0.341} {\scriptsize$\pm$ 0.002} & \textbf{0.409} {\scriptsize$\pm$ 0.005} & \textbf{0.675} {\scriptsize$\pm$ 0.010} & \textbf{0.714} {\scriptsize$\pm$ 0.024} \\
\midrule
\texttt{TGN/TGAT} & w/o KEAT   & 0.435 \scriptsize{$\pm$ 0.069} & 0.396 \scriptsize{$\pm$ 0.060} & 0.313 {\scriptsize$\pm$ 0.012} & 0.349 {\scriptsize$\pm$ 0.020} & 0.607 {\scriptsize$\pm$ 0.014} & 0.586 {\scriptsize$\pm$ 0.037} \\
\texttt{TGN/TGAT} & w/ KEAT    & \textbf{0.515} \scriptsize{$\pm$ 0.053} & \textbf{0.474} \scriptsize{$\pm$ 0.031} & \textbf{0.324} {\scriptsize$\pm$ 0.004} & \textbf{0.380} {\scriptsize$\pm$ 0.007} & \textbf{0.624} {\scriptsize$\pm$ 0.012} & \textbf{0.632} {\scriptsize$\pm$ 0.023} \\
\midrule
\texttt{LeTE} & w/o KEAT & 0.431 \scriptsize{$\pm$ 0.064} & 0.410 \scriptsize{$\pm$ 0.048} &-&-&-&-\\
\texttt{LeTE} & w/ KEAT  & \textbf{0.532} \scriptsize{$\pm$ 0.045} & \textbf{0.483} \scriptsize{$\pm$ 0.042} &-&-&-&-\\
\bottomrule
\end{tabular}
\caption{Ablation study on time encoding strategies with and without KEAT across \texttt{tgbl-wiki}, \texttt{tgbl-review} and \texttt{tgbl-coin} datasets. We report validation and test MRR.}
\label{tab:time_encoding_ablation_2}
\end{table*}

\begin{table*}[h]
\centering

\begin{tabular}{lccccc}
\toprule
\textbf{Component} & \texttt{GraphMixer} & \texttt{TGN/TGAT} & \texttt{LeTE-Fourier} & \texttt{LeTE-Spline} & \texttt{LeTE-50/50} \\
\midrule
Memory                & 172{,}200 & 172{,}400 & 272{,}500 & 262{,}400 & 220{,}250 \\
GAT                   & 67{,}600  & 67{,}800  & 167{,}900 & 157{,}800 & 115{,}650 \\
Link Predictor        & 20{,}301  & 20{,}301  & 20{,}301  & 20{,}301  & 20{,}301 \\
\midrule
Time Encoding (TE)    & 0        & 200       & 100{,}300 & 90{,}200  & 48{,}050 \\
\midrule
\textbf{Total Params} & 260{,}101 & 260{,}301 & 360{,}401 & 350{,}301 & 308{,}151 \\
\textbf{\% More Params} & -0.08\% & --- & +38.46\% & +34.58\% & +18.38\% \\
\bottomrule
\end{tabular}
\caption{Trainable parameters comparison across different time encoding strategies using the TGN backbone for task of link prediction. TGN has three components: memory, GAT, and link predictor, and time encoding (TE) is shared between memory and GAT. LeTE variants differ in how 100-dimensional time encodings are split between Fourier and Spline modeling. $\text{Total params} = \text{Memory} + \text{GAT} + \text{Link Predictor} - \text{Time Encodings}$, as time encoder is shared.}
\label{tab:param-count}
\end{table*}

\begin{figure*}[t]
\centering
\subfloat[\texttt{tgbl-wiki/jodie-wiki}]{\includegraphics[width=0.30\textwidth]{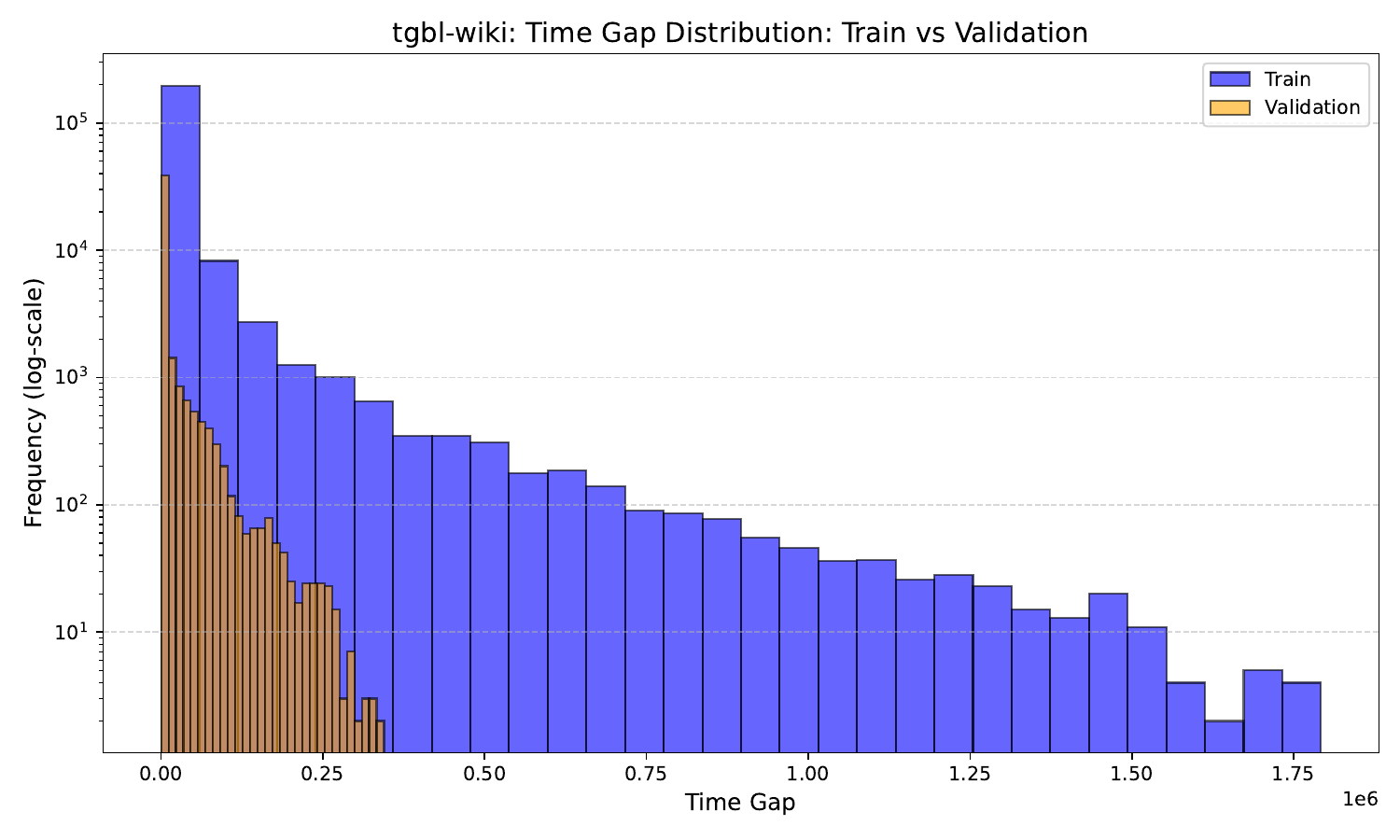}}\hfill
\subfloat[\texttt{tgbl-review}]{\includegraphics[width=0.30\textwidth]{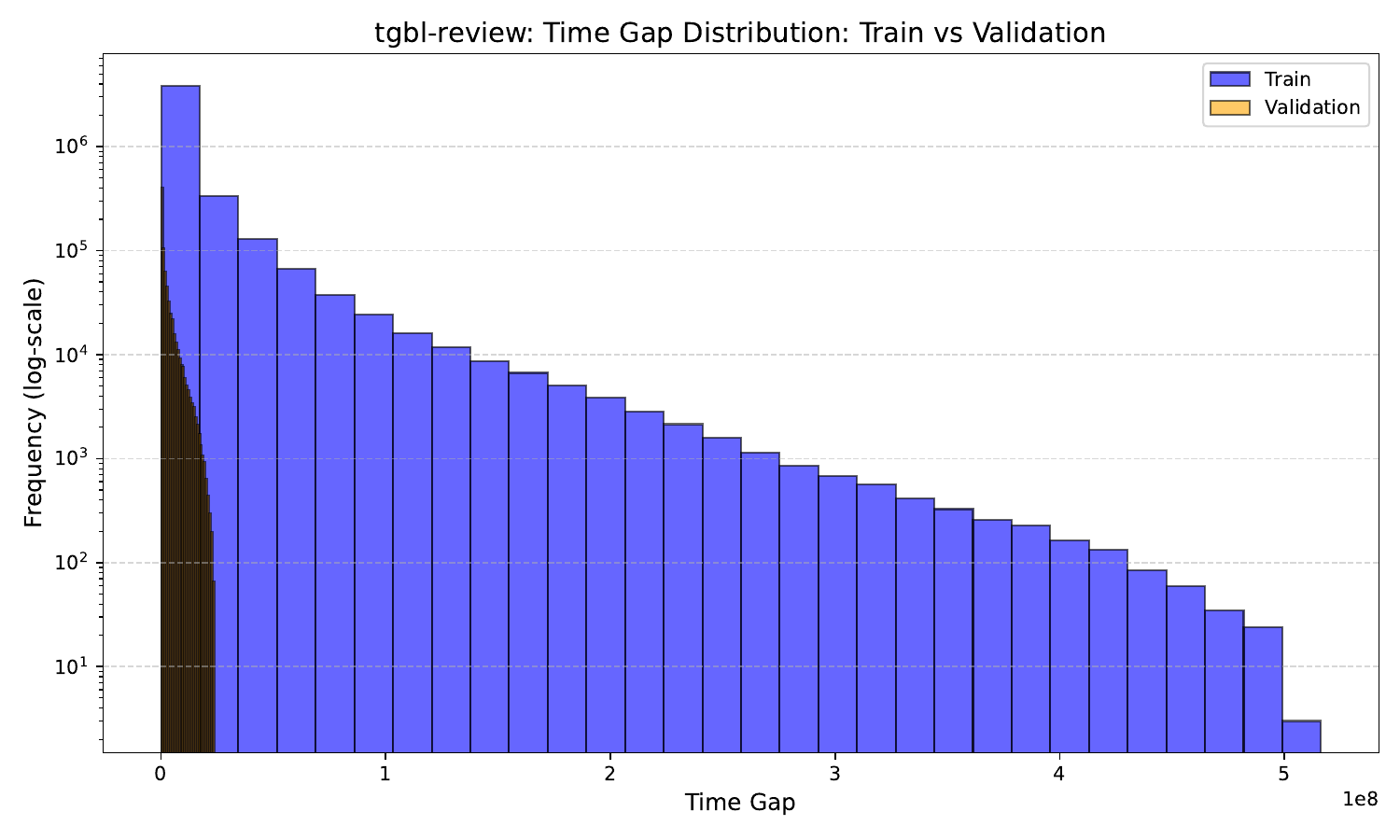}}\hfill
\subfloat[\texttt{tgbl-coin}]{\includegraphics[width=0.30\textwidth]{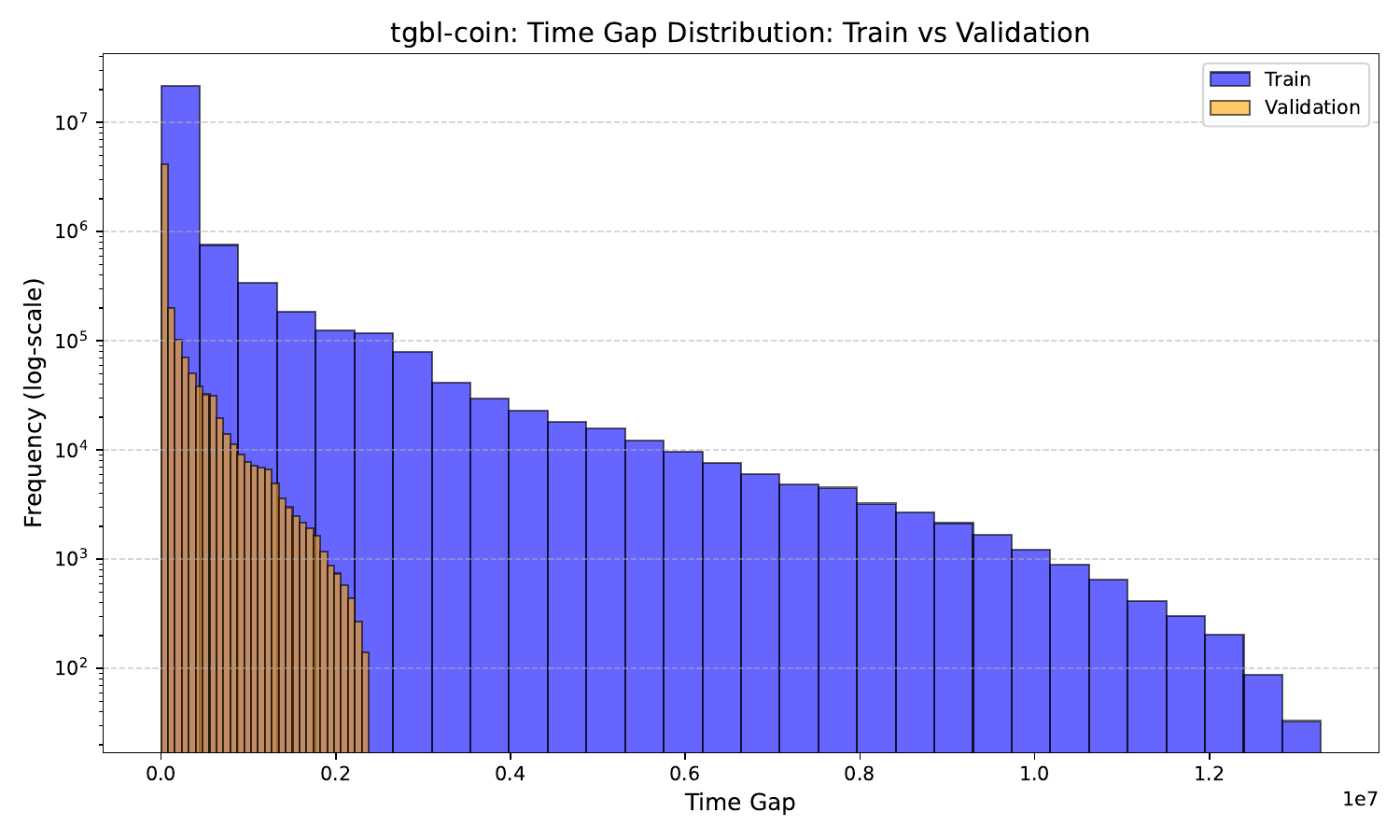}}\\[1ex]

\subfloat[\texttt{tgbl-comment}]{\includegraphics[width=0.30\textwidth]{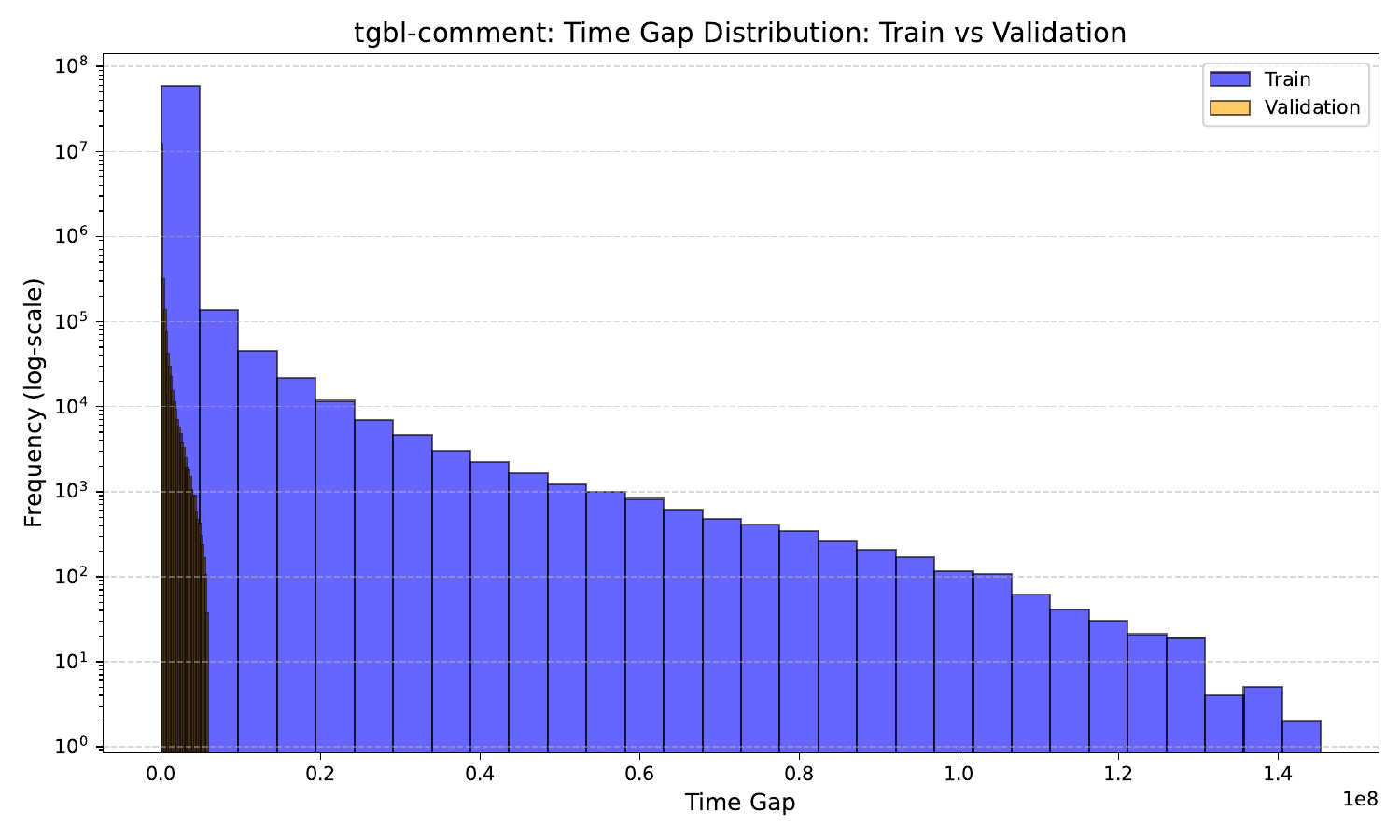}}\hfill
\subfloat[\texttt{tgbl-flight}]{\includegraphics[width=0.30\textwidth]{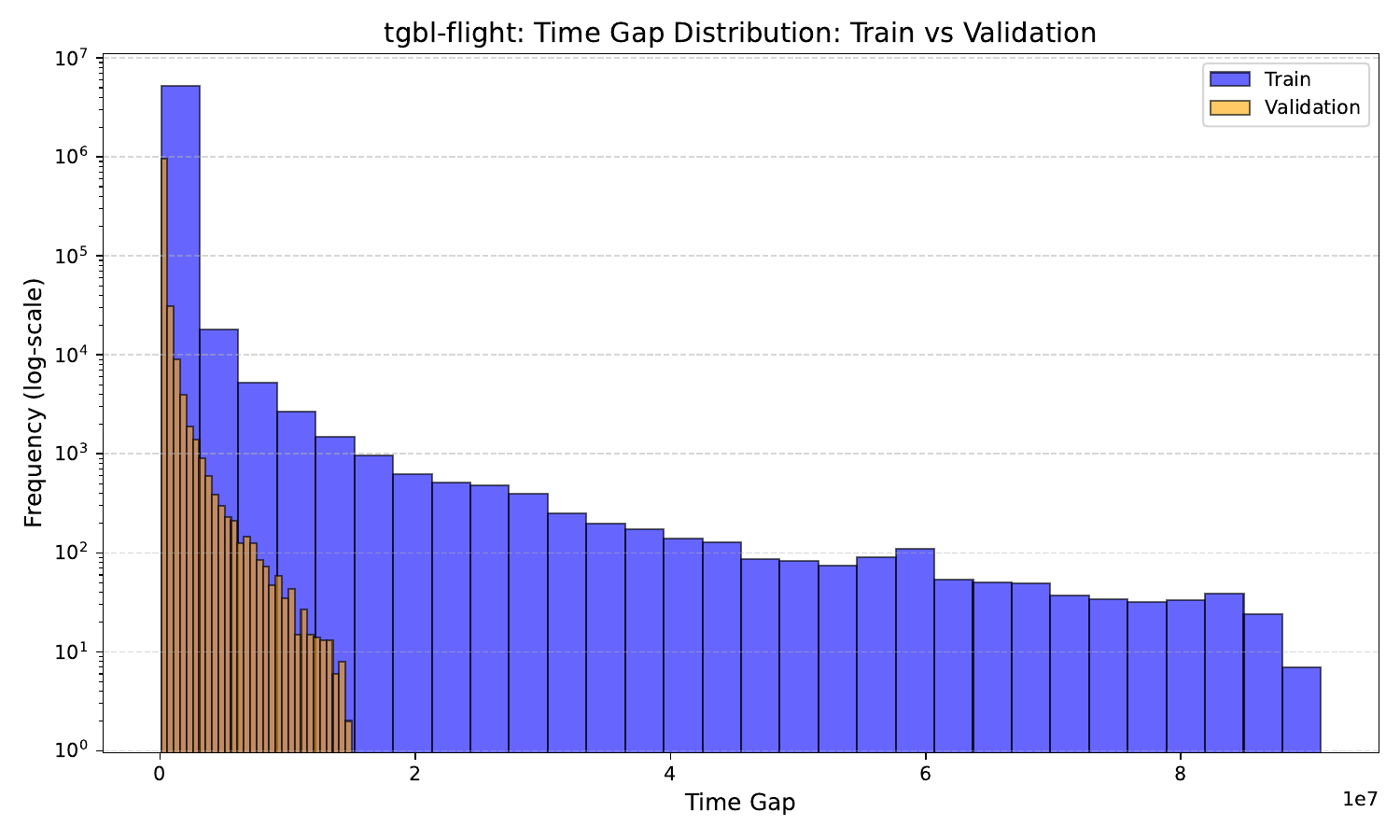}}\hfill
\subfloat[\texttt{tgbn-trade}]{\includegraphics[width=0.30\textwidth]{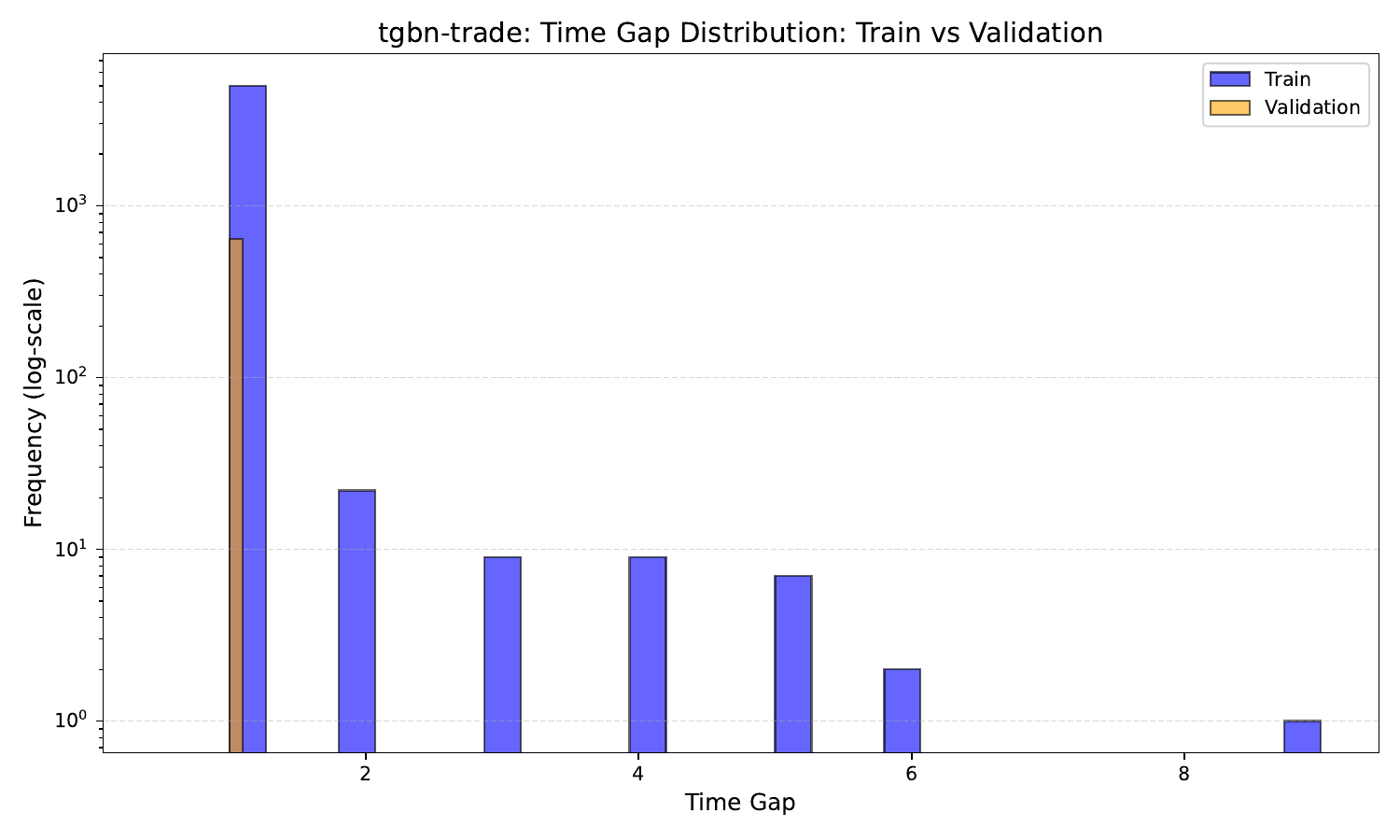}}\\[1ex]

\subfloat[\texttt{tgbn-genre}]{\includegraphics[width=0.30\textwidth]{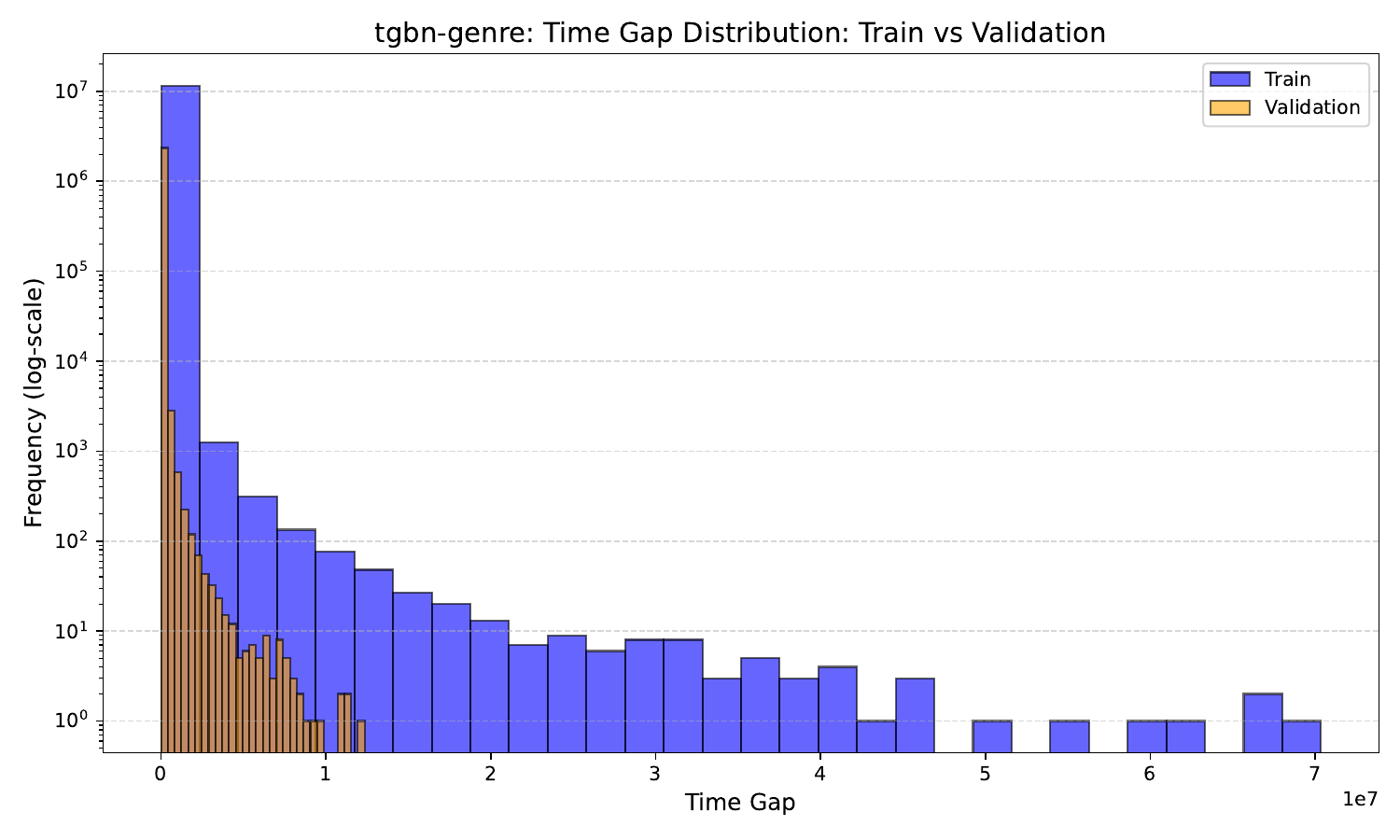}}\hfill
\subfloat[\texttt{tgbn-reddit}]{\includegraphics[width=0.30\textwidth]{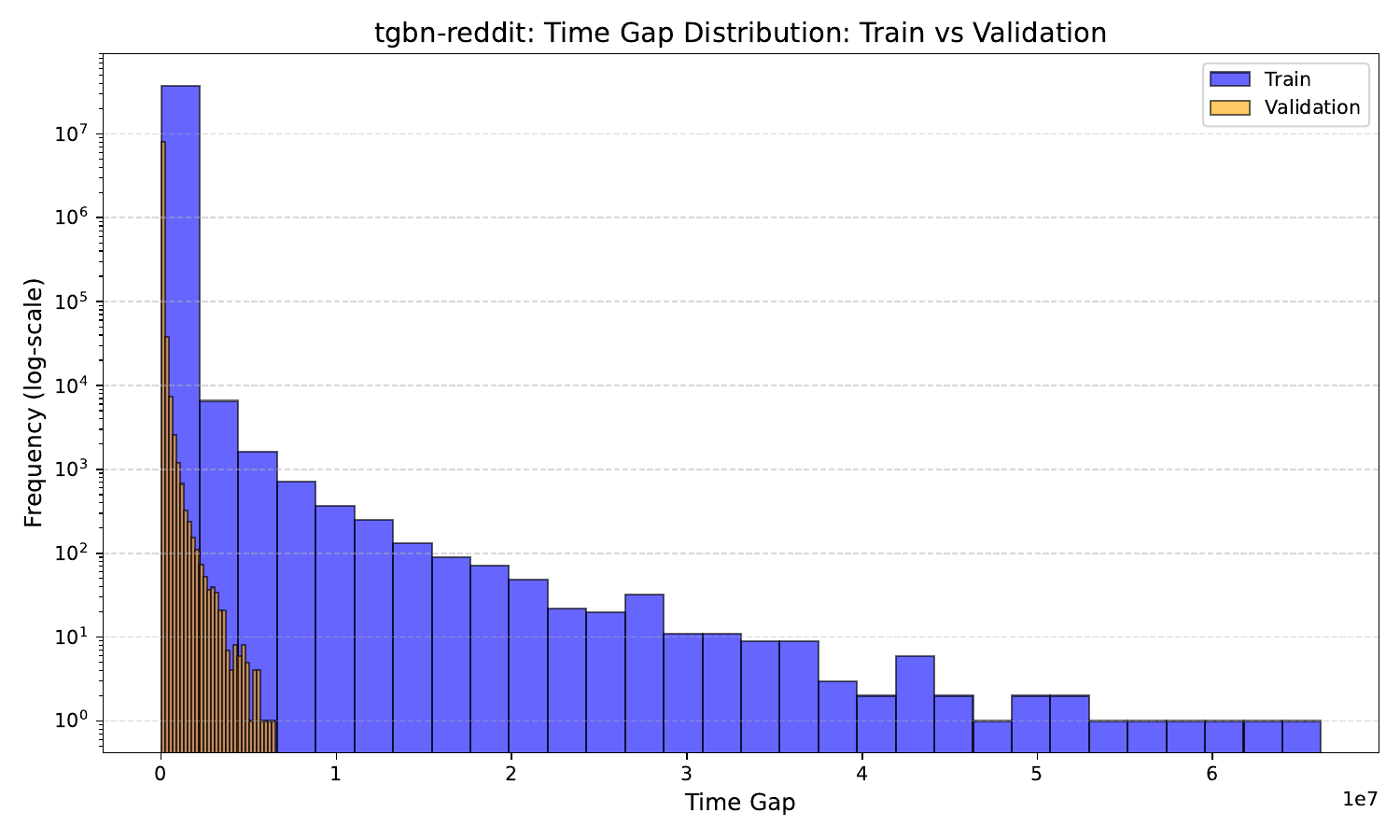}}\hfill
\subfloat[\texttt{tgbn-token}]{\includegraphics[width=0.30\textwidth]{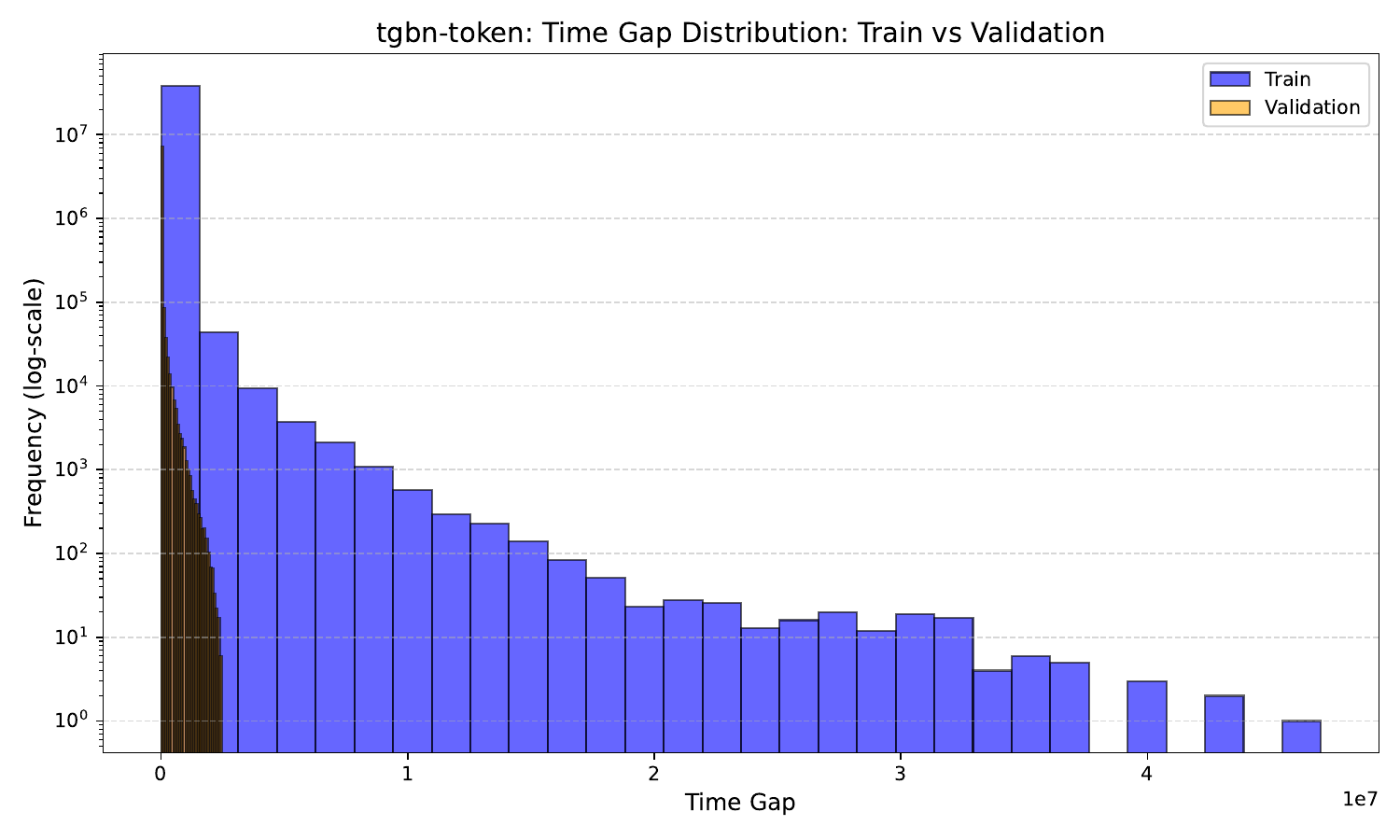}}\\[1ex]

\subfloat[\texttt{jodie-reddit}]{\includegraphics[width=0.30\textwidth]{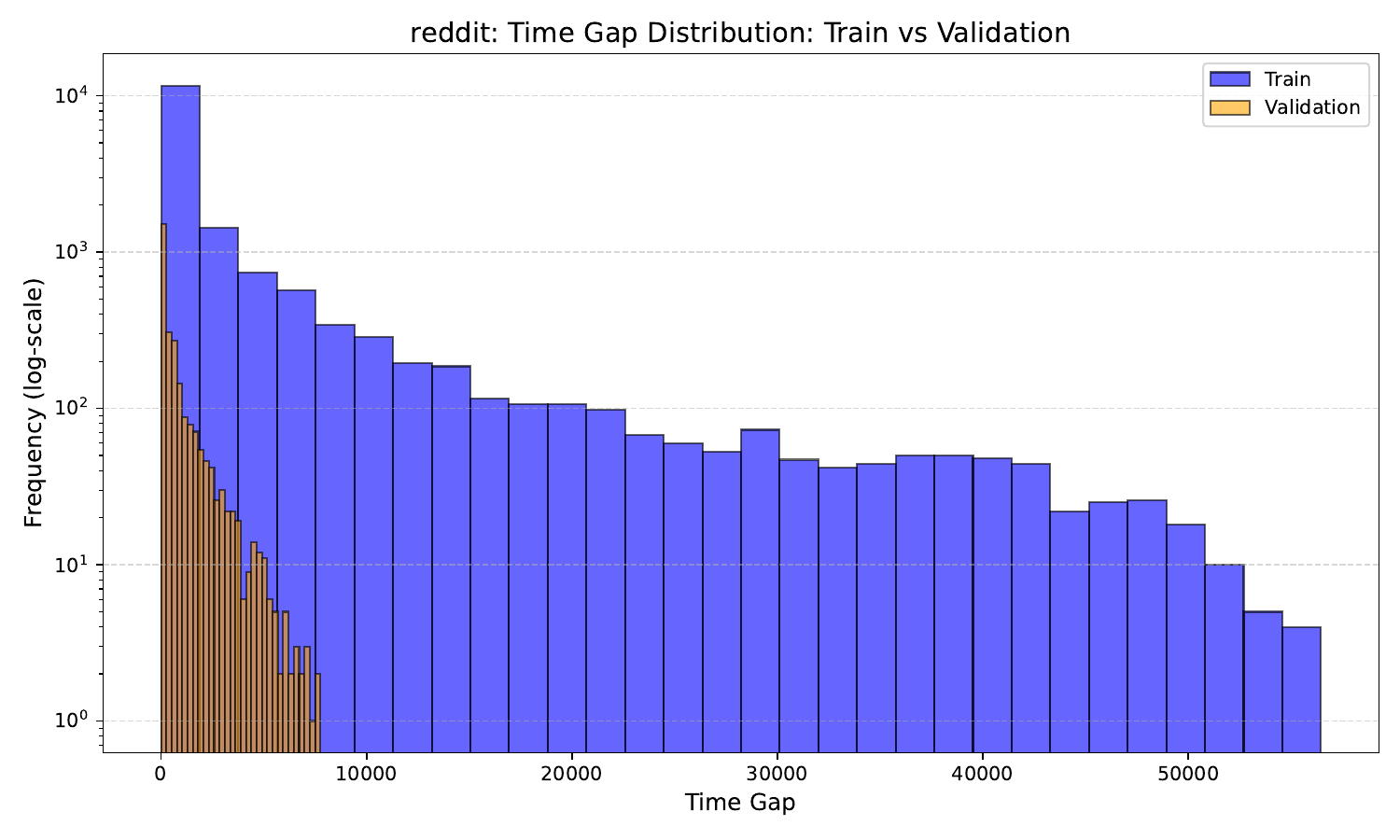}}\hfill
\subfloat[\texttt{jodie-mooc}]{\includegraphics[width=0.30\textwidth]{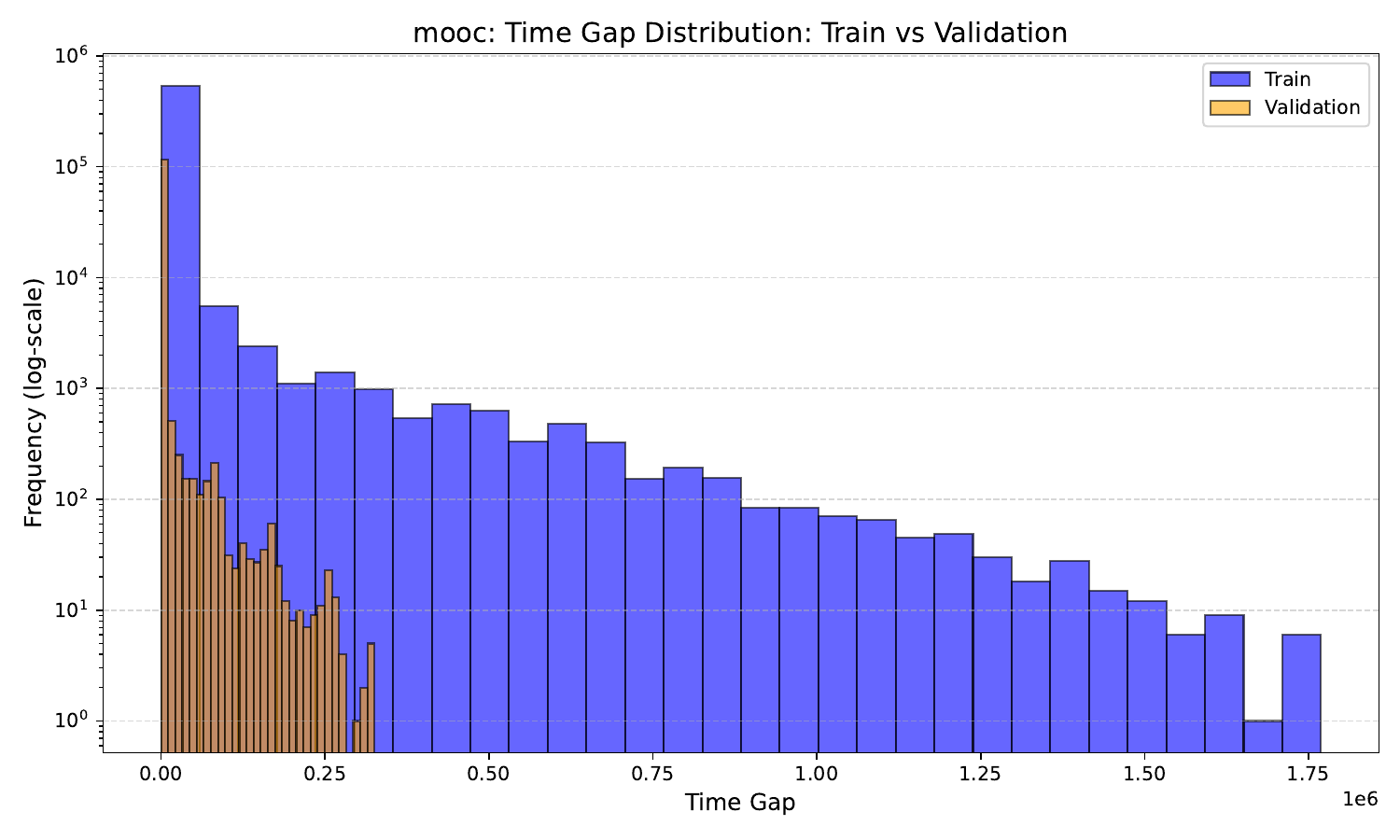}}\hfill
\subfloat[\texttt{jodie-lastfm}]{\includegraphics[width=0.30\textwidth]{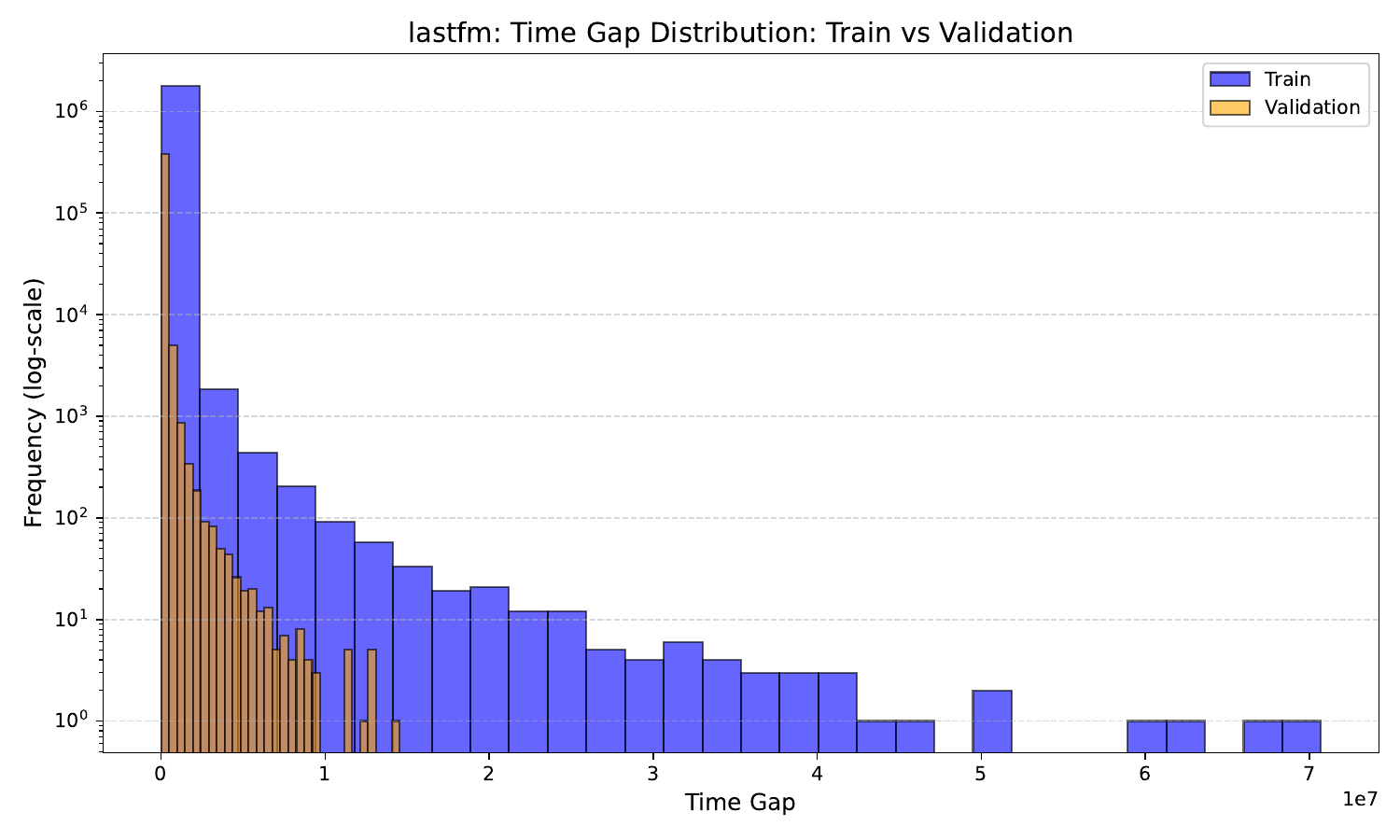}}

\caption{Inter-arrival time distributions $p(\Delta t)$ for 12 temporal graph datasets. Each subplot shows the distribution of time gaps $\Delta t$ between consecutive node-level interactions, computed separately for training and validation sets. Noticeable shifts in $p(\Delta t)$ across these splits indicate temporal distribution drift, which can impact model generalization and motivates the use of kernel-based temporal modulation.}

\label{fig:12_dataset_dist}
\end{figure*}

\begin{figure*}[t]
    \centering
    \includegraphics[width=0.98\linewidth]{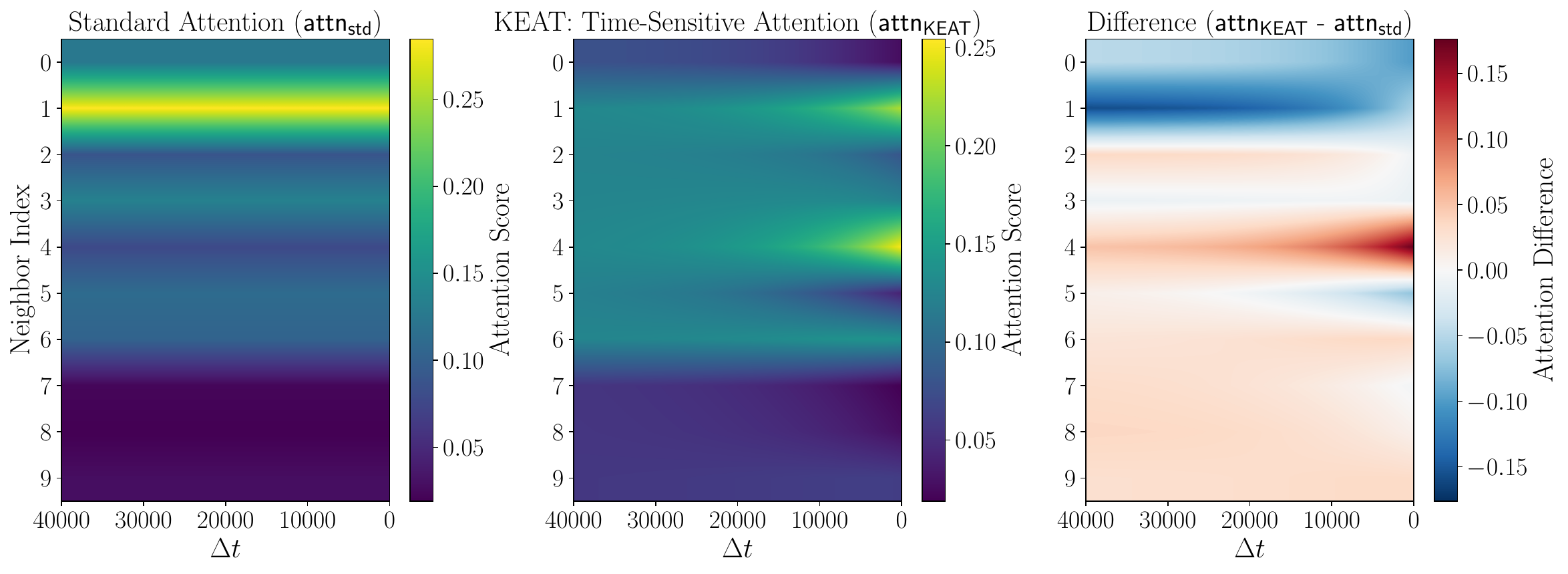}
   \caption{Reproduced temporal attention visualization on \texttt{tgbl-wiki} to support interpretability analysis in Appendix~G. The figure shows attention weights over ten neighbors across relative time $\Delta t$. (Left) Standard attention exhibits \textbf{semantic attention blurring}, assigning nearly uniform weights regardless of temporal distance. (Middle) KEAT introduces time-sensitive edge modulation, assigning higher weights to recent and contextually relevant neighbors. (Right) The difference heatmap illustrates how KEAT provides sharper and temporally aware focus. This figure complements the interpretability discussion by showing how attention evolves with edge timing. Figure \ref{fig:attn_for_10_neighbors_1} shows individual neighbor-wise attentions.}
    \label{fig:attn_summary_1}
\end{figure*}

\begin{figure*}[t]
    \centering
    \includegraphics[width=0.98\linewidth]{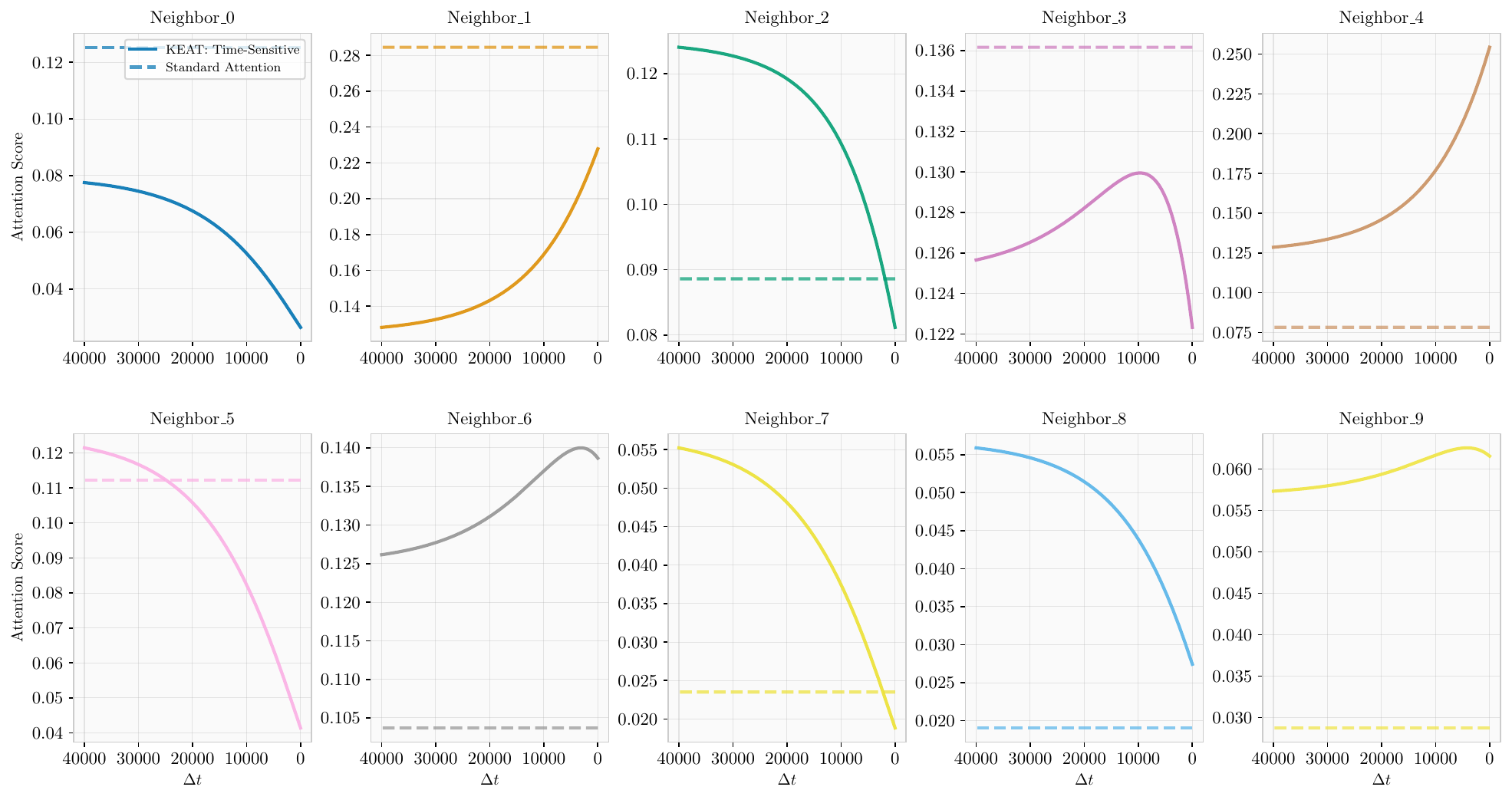}
    \caption{Per-neighbor attention dynamics on \texttt{tgbl-wiki} under time-invariant and time-sensitive settings.
Each subplot shows attention weights assigned to a specific neighbor across relative time $\Delta t$. Lesser value of $\Delta t$ indicates recent interaction. The two curves represent standard attention and KEAT-modulated (time-sensitive) attention. Standard attention looks time-invariant. KEAT dynamically adjusts attention over time, revealing sharper and more context-aware focus compared to the flatter, temporally agnostic patterns of standard attention.}
    \label{fig:attn_for_10_neighbors_1}
\end{figure*}

\begin{figure*}[t]
    \centering
    \includegraphics[width=0.98\linewidth]{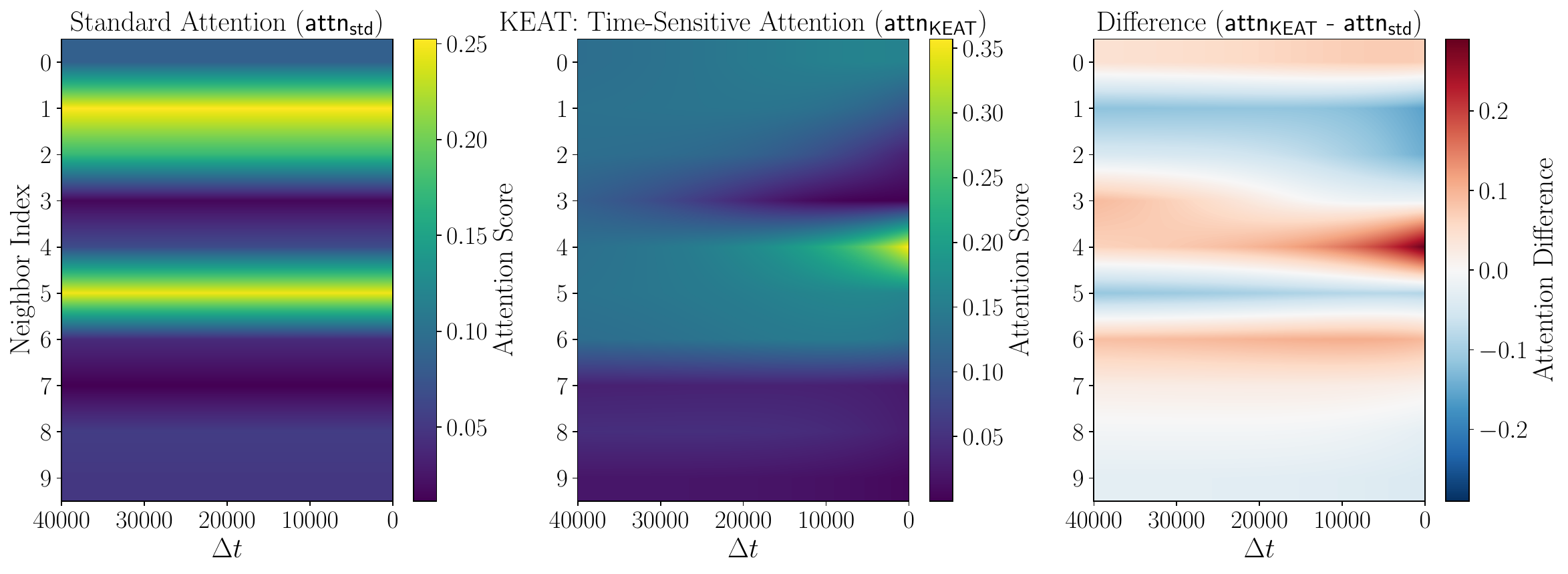}
   \caption{Example 2: Attention visualization, \texttt{tgbl-wiki}.}
    \label{fig:attn_summary_ex_2}
\end{figure*}

\begin{figure*}[t]
    \centering
    \includegraphics[width=0.98\linewidth]{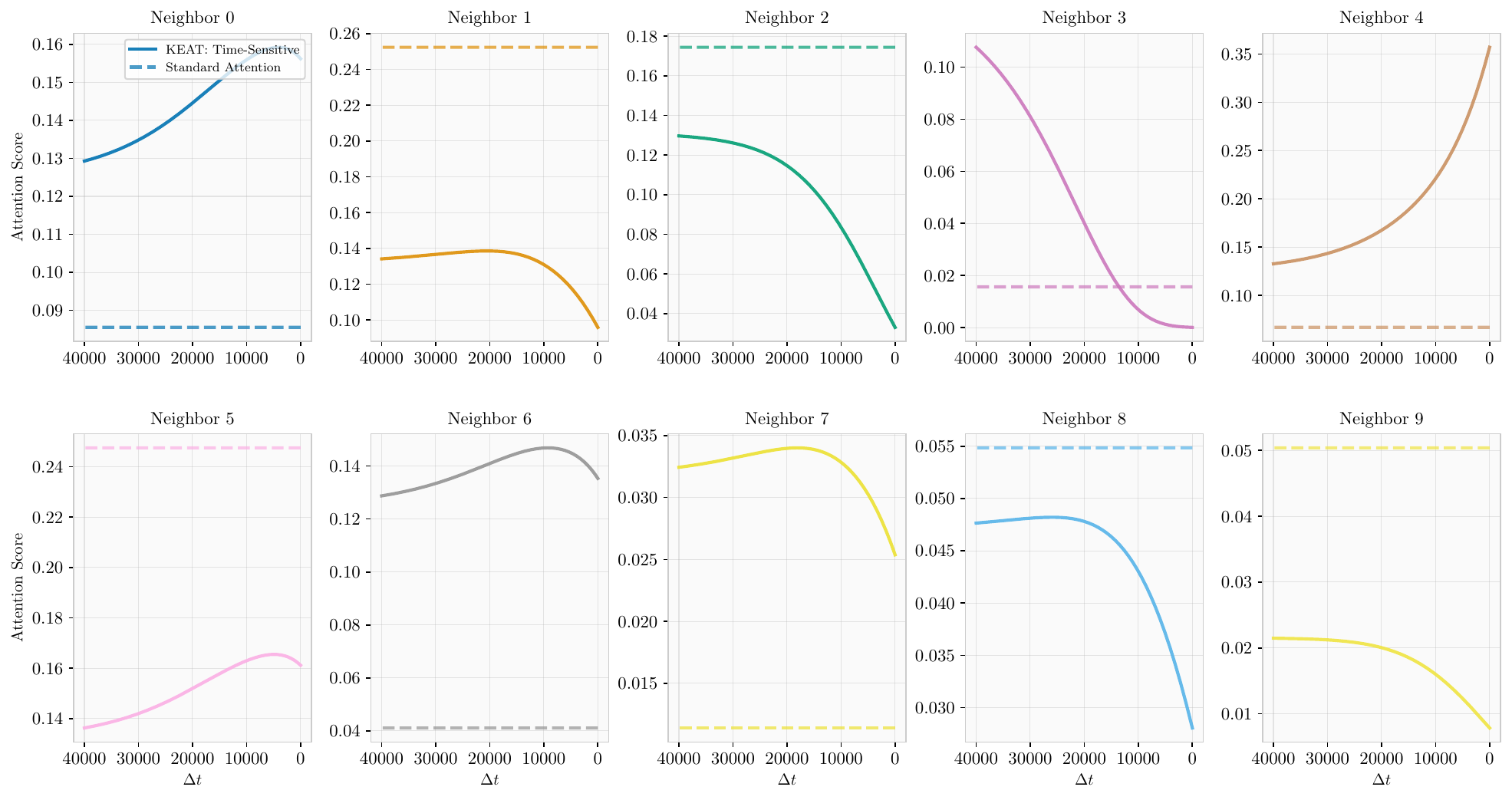}
   \caption{Example 2: Per-neighbor attention visualization, \texttt{tgbl-wiki}.}
    \label{fig:n_summary_ex_2}
\end{figure*}

\begin{figure*}[t]
    \centering
    \includegraphics[width=0.98\linewidth]{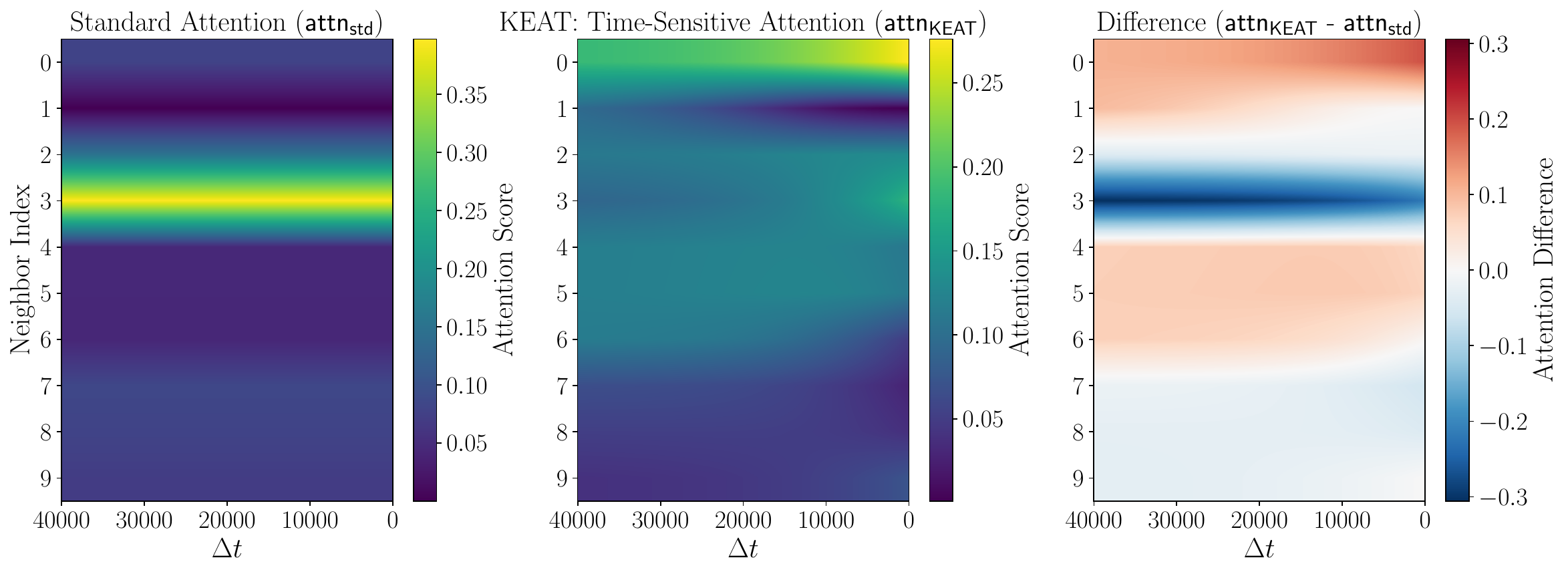}
   \caption{Example 3: Attention visualization, \texttt{tgbl-wiki}.}
    \label{fig:attn_summary_ex_3}
\end{figure*}

\begin{figure*}[t]
    \centering
    \includegraphics[width=0.98\linewidth]{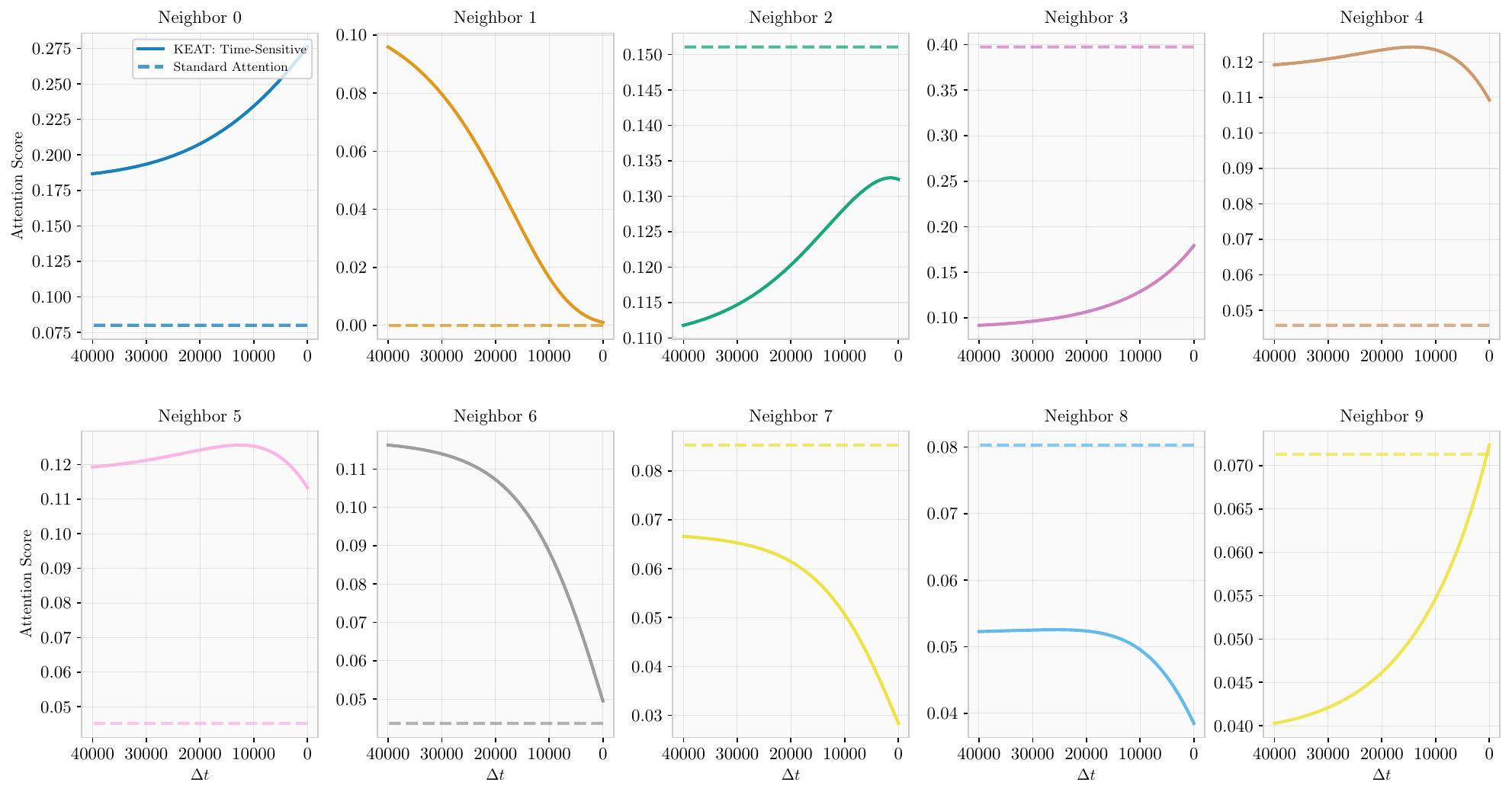}
   \caption{Example 3: Per-neighbor attention visualization, \texttt{tgbl-wiki}.}
    \label{fig:n_summary_ex_3}
\end{figure*}

\end{document}